\documentclass[lettersize,journal]{IEEEtran}

\usepackage{enumitem}
\setlist[itemize]{before=\setlength{\parindent}{10em}}
\usepackage{cite}
\usepackage{color,xcolor}
\usepackage{graphicx}
\usepackage{pdfpages}
\usepackage{tabulary}
\usepackage[ruled,vlined,linesnumbered]{algorithm2e}
\usepackage{amsmath}
\usepackage{amssymb}
\usepackage{float}
\usepackage{soul}
\usepackage{booktabs} 
\usepackage{multirow}  
\usepackage{hyperref} 
\usepackage{array}  
\usepackage{siunitx}

\newcommand{\PreserveBackslash}[1]{\let\temp=\\#1\let\\=\temp}
\newcolumntype{C}[1]{>{\PreserveBackslash\centering}p{#1}}
\newcolumntype{R}[1]{>{\PreserveBackslash\raggedleft}p{#1}}
\newcolumntype{L}[1]{>{\PreserveBackslash\raggedright}p{#1}}

\usepackage{makecell,rotating,multirow,diagbox}
\graphicspath{{figures/}}
\usepackage{bm}
\usepackage{verbatim}

\usepackage{amsfonts}
\usepackage{pifont}

\usepackage{sidecap}

\newcommand{\cyan}{\textcolor{cyan}}
\newcommand{\blue}{\textcolor{blue}}
\newcommand{\red}{\textcolor{red}}
\usepackage{tikz}

\usepackage{mathtools}
\usepackage{sidecap}

\usepackage{subcaption}
\usepackage{booktabs}
\usepackage{hhline}
\usepackage{amsthm}

\newtheorem{theorem}{Theorem}

\newtheorem{assumption}{Assumption}
\newtheorem{remark}{Remark}
\usepackage{pifont}

\begin{document}

\title{PPO-EAL: Exact Augmented Lagrangian Proximal Policy Optimization for Safe Robotic Control}

\author{Jiatao Ding$^{\star}$, Songqun Gao, Andrea Del Prete, and Matteo Saveriano 
	\thanks{All the authors are with the Department of Industrial Engineering, University of Trento, Via Sommarive 9, 38123, Trento, Italy (e-mail:\{jiatao.ding, songqun.gao, andrea.delprete, matteo.saveriano\}@unitn.it). 
		}
	\thanks{The work is supported in part by EU project INVERSE (GA 101136067).}
	\thanks{$\star$ Corresponding author.}
}

\markboth{xx}%
{Shell \MakeLowercase{\textit{et al.}}: A Sample Article Using IEEEtran.cls for IEEE Journals}


\maketitle

\begin{abstract}
Reinforcement learning (RL) has emerged as a promising solution to accomplish complex robotic control tasks; however, most of the current work ignores the safety requirements. Safe RL seeks to maximize task performance while satisfying explicit physical constraints, but current algorithms struggle to learn the policy efficiently with precise constraint satisfaction. This work proposes PPO-EAL, a novel first-order constrained policy optimization framework that integrates exact augmented Lagrangian optimization into proximal policy optimization for safe robotic control. By combining clipped policy updates with exact quadratic penalty terms, PPO-EAL achieves theoretically grounded constraint enforcement without requiring impractically large penalty factors. A momentum-regulated multiplier update further improves dual-variable stability, reducing constraint oscillation and unsafe behavior while preserving task performance. We provide exactness and convergence analysis under standard stochastic approximation assumptions. Extensive validation across diverse GPU-accelerated robotic benchmarks—including cart-pole balancing, cart-double-pendulum stabilization, 7-DoF Franka end-effector reaching, and quadrupedal locomotion—demonstrates superior safety precision and reward performance compared with state-of-the-art first-order safe RL baselines. Finally, we demonstrate zero-shot sim-to-real deployment in a contact-rich gear assembly task, where PPO-EAL substantially improves task success, reduces peak contact force, and enhances operational robustness. These results establish PPO-EAL as a general and practically deployable safe RL framework for diverse safety-critical robotic systems.
\end{abstract}



\section{Introduction} \label{introduction}
\IEEEPARstart{R}{einforcement} learning (RL)~\cite{sutton2018reinforcement} has been widely used to solve complex robotic control tasks, such as manipulation \cite{luo2025precise, falco2018policy,wang2025robot}, locomotion \cite{radosavovic2024real, hoeller2024anymal,atanassov2024curriculum,apostolides2025explosive,zhou2026curriculum}, and multi-robot control \cite{sebastian2025physics,hua2025multi}. However, most current approaches ignore feasibility constraints in the policy updating process, thus failing to satisfy the safety requirement.  
To overcome this issue, {safe RL} has been proposed to find a policy that maximizes the cumulative return while satisfying feasibility constraints, among which the {constrained Markov decision process} (CMDP)~\cite{altman2021constrained,kim2024not} provides a principled framework by introducing auxiliary cost functions in problem formulation. 

An intuitive idea to solve a CMDP is to approximate the nonconvex constrained objective within a local quadratic trust region~\cite{achiam2017constrained, yang2020projection}. 
However, these quadratic methods require solving conjugate-gradient systems involving the inverse of Fisher information matrix, leading to a high computational load. 
First-order Lagrangian-based algorithms, which use a simpler structure, can perform comparably or even better. For instance, the work in~\cite{ray2019benchmarking} constrains the proximal policy optimization (PPO)~\cite{schulman2017proximal} with Lagrangian methods that use adaptive penalty coefficients to enforce constraints. In~\cite{zhang2020first}, a Lagrangian argument is used to find the Karush–Kuhn–Tucker (KKT) points as the optimum of the constrained local policy optimization. Nevertheless, theoretical analyses~\cite{tessler2018reward,paternain2019constrained} reveal that the updates of the Lagrange multipliers are slower than those of the policy parameters, which are prone to cause oscillations or overshooting around the cost boundary. The approach in~\cite{stooke2020responsive} alleviates this problem by introducing a PID-type multiplier adjustment, though it is hard to tune due to the sensitivity to three hyperparameters. In contrast to Lagrangian-based methods, alternative constrained RL approaches include both barrier-based and penalty-based formulations. Barrier-based methods, such as interior-point policy optimization (IPO)~\cite{liu2020ipo}, employ logarithmic barrier functions to maintain feasibility, but their surrogate objectives introduce bias and often lead to suboptimal solutions~\cite{freund2004penalty}. Penalty-based methods instead directly penalize constraint violations; for example, penalized proximal policy optimization (P3O)~\cite{zhang2022penalized} adopts an $\ell_1$-based exact penalty formulation, while~\cite{lee2024exploring} demonstrates that cost normalization can improve empirical performance. However, such approaches require very large or infinite penalty factors to guarantee equivalence with the original constrained problem~\cite{freund2004penalty}, often resulting in ill-conditioned optimization and numerical instability.



To address these limitations, we propose the exact augmented Lagrangian PPO (PPO-EAL), a constrained PPO method which augments the Lagrangian function of the primal CMDP with a quadratic deviation term, using an exact formulation. This {multiplier–penalty} construction stabilizes cost dynamics and remains mathematically equivalent to the original constrained problem without requiring an infinite penalty factor. 
Inspired by PPO~\cite{schulman2017proximal}, PPO-EAL adopts a clipped surrogate objective for both reward and cost, thereby avoiding expensive trust-region optimization. Differing from the augmented PPO (APPO) method in~\cite{dai2023augmented}, we do not introduce the relaxation values in the problem formulation, and we update policy parameters and Lagrangian multipliers at different time scales to guarantee convergence. In addition, we stabilize the Lagrangian multipliers update by adding momentum, which achieves better performance with lower oscillations.
Differing from previous benchmark work such as~\cite{ray2019benchmarking} and~\cite{dai2023augmented}, which focus on the task with one constraint, we highlight effectiveness for more complex tasks with multiple physical constraints by validating PPO-EAL across a broad suite of increasingly complex robotic benchmarks—including dynamic balancing, manipulator control, and quadrupedal locomotion. Finally, we demonstrate zero-shot sim-to-real deployment in a real-world contact-rich gear assembly task, highlighting safety, robustness, and practical applicability for safety-critical robotic systems.

Our main contributions are summarized as follows:
\begin{itemize}
    \item We propose a novel algorithm, {PPO-EAL}, that integrates Lagrangian duality with an exact quadratic penalty, achieving provable exactness and convergence.
    \item We implement the safe RL pipeline, together with a momentum-based modification, under the deep actor–critic architecture using the standard first-order optimizer, and empirically demonstrate superior performance over state-of-the-art safe RL baselines.
    \item We extensively validate PPO-EAL across a broad suite of increasingly complex robotic control benchmarks with multiple physical constraints, and further demonstrate zero-shot sim-to-real deployment in a contact-rich gear assembly task, establishing its practical effectiveness for safety-critical robotic systems.    
\end{itemize}

The rest of the work is organized as follows: Section~\ref{background} introduces the background for CMDP. Section~\ref{PPO_EAL_methodology} formulates PPO-EAL pipeline, with provable exactness and convergence. Section~\ref{evaluation_exp} evaluates the method in robotic control tasks, and Section~\ref{conclu_sec} concludes our work.

\section{Background}\label{background}

RL models a control problem as a Markov Decision Process (MDP), which is a tuple $<\!\mathcal{S}, \mathcal{A}, r, p, \mu\!>$. Here, $\mathcal{S}$ is the set of states, $\mathcal{A}$ is the set of actions, $r: \mathcal{S} \times \mathcal{A} \times \mathcal{S} \to \mathbb{R}$ is the reward function, $p: \mathcal{S} \times \mathcal{A} \times \mathcal{S} \to [0,1]$ is the state transition probability (where $P(s'|s,a)$ is the transition probability from state $s$ to
state $s'$ given action $a$), and $\mu$ is the initial state distribution. To solve an MDP, we aim to find a policy $\pi(a|s): \mathcal{S} \rightarrow \mathcal{P}(\mathcal{A})$ that maximizes
\begin{equation}\small
J_{R}(\pi) = \mathbb{E}_{\tau \sim \pi}\Biggl[\sum_{t=0}^{\infty} \gamma^{t}r(s_{t}, a_{t}, s_{t+1})\Biggr],
\label{eq:reward_return}
\end{equation}
where
\(\tau = \bigl(s_{0},a_{0},s_{1},a_{1},\ldots\bigr)\) denotes a trajectory, and \(\tau \sim \pi\) means that the distribution over trajectories
is determined by policy \(\pi\), i.e.,
$s_{0} \!\sim \!\mu,\quad
a_{t} \!\sim \!\pi\bigl(\cdot \mid s_{t}\bigr),\quad
s_{t+1} \!\sim \!P\bigl(\cdot \mid s_{t},a_{t}\bigr)$.  \(\gamma \in [0,1)\) is the discount factor. The expectation $\mathbb{E}[\cdot]$ represents the empirical average over a batch of sampled trajectories. The initial state $s_{0}$ is sampled from $\mu$, and subsequent transitions $(s_{t}, a_{t}, s_{t+1})$ are generated by following policy $\pi$. 

Following the general principle, we denote the value function as
{\small $V^{\pi}(s) \!:=\! \mathbb{E}_{\tau \sim \pi}\!\biggl[\sum_{t=0}^{\infty} \gamma^{t} R(s_{t}, \!a_{t}) \bigm| s_{0} \!\!=\!\! s\!\biggr]$} and the action‐value function as
{\small$
Q^{\pi}(s, a) := \mathbb{E}_{\tau \sim \pi}\biggl[\sum_{t=0}^{\infty} \gamma^{t} R(s_{t}, a_{t}) \bigm| s_{0} = s, a_{0} = a\biggr]$}. The advantage function is defined as
$A^{\pi}(s, a) := Q^{\pi}(s, a) - V^{\pi}(s).$

To address constrained problems, the above framework is extended as a CMDP. The MDP is augmented with a set $\mathcal{C}=\{c_{1},\ldots, c_{n}\}$ of cost functions  that capture constraint violations, and corresponding thresholds $\mathcal{B} = \{b_{1}, \ldots, b_{n}\}$~\cite{altman2021constrained,achiam2017constrained}. Each cost $c_{i} : \mathcal{S}\times\mathcal{A}\times\mathcal{S} \to \mathbb{R}$ maps state-action-state triplets to a non-negative cost. In the constrained setting, an optimal policy maximizes the expected discounted return \eqref{eq:reward_return}, while ensuring that each discounted-sum cost $J_{C_{i}}(\pi)$ remains below its threshold $b_{i}$:
{\small
\begin{align}
\underset{\pi}{\mathrm{maximize}}  \quad & J_{R}(\pi) \label{eq:cmdp_obj} \\
\mathrm{s.t.} \quad & J_{C_{i}}(\pi) \le b_{i}, \quad \forall i \in \{1, \ldots, n\},
\label{eq:cmdp_constraints}
\end{align}
}
with
{\small
\begin{align}
\quad J_{C_{i}}(\pi) = \mathbb{E}_{\tau \sim \pi}\Biggl[ \sum_{t=0}^{\infty} \gamma^{t} c_{i}(s_{t}, a_{t}, s_{t+1}) \Biggr].
\label{eq:cmdp_specify}
\end{align}
}

Using the performance difference lemma \cite{kakade2002approximately}, and defining the discounted future state visitation distribution as
$d^{\pi}(s) := (1 - \gamma)\sum_{t=0}^{\infty} \gamma^{t} P\bigl(s_{t} = s \mid \pi\bigr)$, one can derive a surrogate form of \eqref{eq:cmdp_obj}--\eqref{eq:cmdp_constraints} in terms of advantage functions over two policies ($\pi$ and $\pi_{\text{old}}$) as
{\small
\begin{align*}
J_{R}(\pi) - J_{R}(\pi_{\text{old}}) &= \frac{1}{1-\gamma}\mathbb{E}_{(s,a)\sim (d^{\pi},\pi)}\bigl[A^{\pi_{\text{old}}}_{R,t}(s,a)\bigr] \\
J_{C_{i}}(\pi) - J_{C_{i}}(\pi_{\text{old}}) &= \frac{1}{1-\gamma} \mathbb{E}_{(s,a)\sim (d^{\pi},\pi)}\bigl[A^{\pi_{\text{old}}}_{C_{i},t}(s,a)\bigr], \quad \!\!\! \forall i,
\end{align*}
}%
where $A^{\pi}_{R,t}(s,a)$ denotes the reward advantage function and $A^{\pi}_{C_{i},t}(s,a)$ denotes the cost advantage for the $i$-th constraint at time-step $t$. 
%
Then, we can formulate the unbiased constrained policy optimization problem as 
\begin{equation}\label{eq:primal_problem}\small
\begin{aligned}
\pi^{\star} &= \underset{\pi}{\mathrm{arg\,max}} ~ \mathbb{E}_{(s,a)\sim (d^{\pi},\pi})\bigl[A^{\pi_\text{old}}_{R,t}(s,a)\bigr] \\
\mathrm{s.t.} &~ J_{C_{i}}(\pi_\text{old}) \!+\! \frac{1}{1-\gamma}\underset{(s,a)\sim (d^{\pi},\pi)}{\mathbb{E}}\bigl[A^{\pi_\text{old}}_{C_{i},t}(s,a)\bigr] \!\leq \!b_{i}, \quad \!\!\! \forall i.  
\end{aligned} 
\end{equation}

\section{PPO with exact augmented Lagrangian}\label{PPO_EAL_methodology}
To solve the CMDP problem \eqref{eq:primal_problem} with smooth convergence, we here augment the primal problem with an exact augmented Lagrangian function.

\subsection{Exact Augmented Lagrangian} 
For brevity, let us denote $\phi_i(\pi) = J_{C_i}(\pi) - b_i$. 
Then, we solve the above constrained PPO with the augmented Lagrangian function ${L}_{\text{EAL}}(\pi,\lambda)$, following
\begin{equation}\small
\begin{aligned}
  	{L}_{\text{EAL}}(\pi,\lambda) \!=\! J_{R}(\pi) \!- \!\sum_{i=1}^{n}\lambda_{i}\phi_i(\pi)\!-\! \frac{\beta}{2} \sum_{i=1}^{n}\bigl([\phi_i(\pi)]_{+})^2,
	\label{eq:ext_alagrangian}  
\end{aligned}
\end{equation}
with $\lambda_{i}\geq 0$ being the dual variables, $\beta \geq 0$ denoting the penalty coefficient of the quadratic term, and $[\cdot]_{+}$ denoting the ReLU function.   

In a real application, we parameterize the policy $\pi$ by $\theta$. Then, a {naive} primal-dual scheme alternates
\begin{align}\small
	\theta' &\!= \!\theta \!+ \!\alpha_{\theta}\nabla_{\theta}\Bigl[L^\text{CLIP}_{R}(\theta) \!-\!\! \sum_{i=1}^{n} \lambda_{i}L^\text{CLIP}_{C_{i}}(\theta)\!- \!\frac{\beta}{2}\!\sum_{i=1}^{n} [L^\text{CLIP}_{C_{i}}(\theta)]_{+}^2\Bigr], 
	\label{eq:pdd_update_theta}\\
	\lambda'_{i} &\!=\! \max\bigl\{0,\lambda_{i} \!+\! \alpha_{\lambda_{i}}\phi_i(\pi_{\theta}) \bigr\}, 
	\label{eq:exa_pdd_update_lambda}
\end{align}
where $\alpha_{\theta}$ and $\alpha_{\lambda_{i}}$ are the learning rates for policy parameters and Lagrangian multipliers, respectively\footnote{In practice, we use one fixed $\alpha_{\lambda}$ for all constraints.}.  
To avoid aggressive policy change after each iteration, we use the clipped surrogates for reward and cost (denoted as $L^\text{CLIP}_{R}$ and $L^\text{CLIP}_{C_{i}}$) following the original PPO pipeline, which are defined as
\begin{align}
	L^\text{CLIP}_{R}(\theta) &= -\mathbb{E}_t \Bigl[\min\bigl(r_t(\theta)\hat{A}^{\pi}_{R,t},  \mathrm{clip}\label{eq:ppo_clip_reward}(r_t(\theta))\hat{A}^{\pi}_{R,t}\bigr)\Bigr],\\
	L^\text{CLIP}_{C_{i}}(\theta) &= \mathbb{E}_t \Bigl[\max\bigl(r_t(\theta)\hat{A}^{\pi}_{C_{i},t},  \mathrm{clip}(r_t(\theta))\hat{A}^{\pi}_{C_{i},t}\bigr)\Bigr],
	\label{eq:ppo_clip_cost}
\end{align}
where 
$r_t(\theta) = \frac{\pi_{\theta}(a_t \mid s_t)}{\pi_{\theta_\text{old}}(a_t \mid s_t)}$ is the importance sampling ratio, $\hat{A}^{\pi}_{R,t}$ and $\hat{A}^{\pi}_{C_i,t}$ are the normalized reward advantage and cost advantage at time $t$,
and $\mathrm{clip}(r_t(\theta)) = \mathrm{clip}\bigl(r_t(\theta),1-\delta,1+\delta\bigr)$ for some $\delta>0$. The advantages $\hat{A}^{\pi}_{R,t}$ and $\hat{A}^{\pi}_{C_i,t}$ are estimated from past trajectories and value networks ($V^{\pi}_{R,t}$ and $V^{\pi}_{C_i,t}$) using the generalized advantage estimation~\cite{schulman2015high}.

\subsection{Exactness and Convergence Analysis}
Under standard regularity assumptions, PPO-EAP admits exactness and convergence properties. The exactness property follows from the finite-penalty exact augmented Lagrangian theory~\cite{bertsekas2014constrained}. The proof of convergence is standard for stochastic approximation (SA) algorithms \cite{bhatnagar2009natural,bhatnagar2013stochastic}, which is then well formulated in \cite{chow2018risk} for risk–constrained policy–gradient and actor–critic algorithms.

\subsubsection{Exactness analysis}
To state the exactness, we make the following assumptions:
\begin{assumption}\label{ass:continuity-reward}
Each $J_R(\pi)$ and $J_{C_i}(\pi)$ is continuously differentiable, with Lipschitz–continuous gradients.
\end{assumption}
\begin{assumption}\label{ass:feasiblity}
The feasible set $\{\pi \mid \phi_i(\pi) \le 0,\, \forall i\}$ is nonempty, and compactness is enforced through projection.
\end{assumption}

\begin{theorem}[Exactness]\label{thm:exactness}
Under Assumptions \ref{ass:continuity-reward}--\ref{ass:feasiblity}, there exists a finite $\beta^\star > 0$
and a neighborhood $\mathcal{N}(\lambda^\star)$ of the optimal multipliers such that
for all $\beta \ge \beta^\star$ and $\lambda \in \mathcal{N}(\lambda^\star)$:
(i) if $\pi$ is feasible ($\phi_i(\pi) = J_{C_i}(\pi) \leq 0, \lambda_i \geq 0$) and satisfies the stationarity condition $0 \in \partial_{\pi} L_{\mathrm{EAL}}(\pi, \lambda)$, then $\pi$ is a KKT point (and hence a local solution) of the original constrained problem~\eqref{eq:primal_problem};
(ii) conversely, any KKT point $(\pi^\star, \lambda^\star)$
is a stationary point of $L_{\text{EAL}}$ in~\eqref{eq:ext_alagrangian}.
\end{theorem}

\noindent
\textit{Proof sketch:} The exactness result (Theorem~1) follows directly from the finite--penalty exactness theorem of the Rockafellar augmented Lagrangian 
(\cite{rockafellar1976augmented}; see also \cite{bertsekas2014constrained}, Chapter~4). 

To make it clear, we introduce nonnegative slack variables $p_i \ge 0$ and use the inf--projection identity
\[
[\phi_i(\pi)]_+^2 = \min_{p_i \ge 0} (\phi_i(\pi) + p_i)^2 ,
\]
so~\eqref{eq:ext_alagrangian} can be rewritten as
\begin{align} \nonumber
L_{\mathrm{EAL}}(\pi, \lambda)
= \min_{p \ge 0} 
\Big\{
&J_R(\pi)
- \sum_i \lambda_i (\phi_i(\pi) + p_i)\nonumber
\\&- \frac{\beta}{2} \sum_i (\phi_i(\pi) + p_i)^2
\Big\}.\nonumber
\label{eq:eal_min}
\end{align}

Define the equality constraints
\[
h_i(\pi,p) := \phi_i(\pi) + p_i = 0,
\qquad p_i \ge 0.
\]
Hence, the inequality--constrained CMDP is reformulated as an equality--constrained problem in $(\pi,p)$.
The above augmented Lagrangian is precisely the standard 
Rockafellar augmented Lagrangian for equality constraints:
\[
\mathcal{L}_\beta(\pi,p,\lambda)
= J_R(\pi)
- \sum_{i=1}^n \lambda_i h_i(\pi,p)
- \frac{\beta}{2} \sum_{i=1}^n h_i(\pi,p)^2.
\]

By the finite--penalty exactness theorem of the Rockafellar augmented Lagrangian \cite{rockafellar1976augmented}, there exists a finite $\beta^* > 0$ and a neighborhood
$\mathcal{N}(\lambda^*)$ of the optimal multipliers such that, whenever $\beta \ge \beta^*$ and $\lambda \in \mathcal{N}(\lambda^*)$: (i) any stationary point $(\pi,p)$ of $\mathcal{L}_\beta$ with
$h_i(\pi,p)=0$ corresponds to a KKT point of the original constrained problem; and (ii) conversely, any KKT point $(\pi^*,p^*,\lambda^*)$ is a stationary point of $\mathcal{L}_\beta$.

For a feasible $\pi^*$, the optimal slack satisfies $p_i^* = \max\{0,-\phi_i(\pi^*)\}=0$, so we obtain exactly the statement of Theorem~1 for $(\pi^*,\lambda^*)$.
\hfill$\square$

Note that the slack variables $p_i$ are introduced \emph{only for theoretical analysis} to rewrite the inequality constraints as equality constraints, which enables a direct application of the classical exact augmented Lagrangian theory.  In implementation, we do \emph{not} optimize over $p_i$. The minimization of $p_i \ge 0$ admits a closed-form solution and yields the hinge penalty 
$[\varphi_i(\pi)]_{+}^{2}$, which is exactly the term used in the algorithm.

\subsubsection{Convergence analysis}
We now establish the convergence of PPO--EAL under standard SA assumptions.
\begin{assumption}\label{ass:stepsizes}
Step sizes $\alpha_k$ and $\omega_k$ satisfy the Robbins--Monro conditions:
        $\sum_k \alpha_k = \sum_k \omega_k = \infty$,
        $\sum_k (\alpha_k^2 + \omega_k^2) < \infty$,
        and $\omega_k / \alpha_k \to 0$.
\end{assumption}
\begin{assumption}\label{ass:noise}
Gradient estimates are unbiased with martingale–difference noise and bounded variance. That is, 
\[
\mathbb{E}[\xi_{k+1}\!\!\mid\!\!\mathcal{F}_k] \!\!= \!\!0, 
\mathbb{E}[\zeta_{k+1}\!\!\mid\!\!\mathcal{F}_k]\!\! =\!\! 0,
\mathbb{E}[\|\xi_{k+1}\|^2 \!\!+\!\! \|\zeta_{k+1}\|^2 \!\!\mid \!\!\mathcal{F}_k] \!\!<\!\! \infty,
\]
where $(\xi_{k+1},\zeta_{k+1})$ are martingale--difference noise terms, and $\mathcal{F}_k$ denotes the $\sigma$--algebra generated by all information up to iteration $k$.
\end{assumption}
\begin{assumption}[PPO specific]\label{ass:ppo-bias}
The surrogate objective, advantage estimator, and clipping bias introduce bounded perturbations of order $\mathcal{O}(\varepsilon)$, where $\varepsilon$ quantifies the approximation and clipping bias in PPO.
\end{assumption}

Before analyzing the limiting dynamics, we impose the following standard condition on the projection operators.

\begin{assumption} [Projected dynamics]\label{ass:proj_biase} The projection operators $\Gamma_\pi$ and $\Gamma_\lambda$ project the iterates onto compact convex sets. The induced projected vector fields are continuous and locally Lipschitz on these sets.

For a vector field $g(x)$ and projection operator $\Gamma$, we denote the induced projected vector field by
\[
\hat{\Gamma}(g(x))
:=
\lim_{\eta\downarrow 0}
\frac{\Gamma(x+\eta g(x))-x}{\eta}.
\]
\end{assumption}

\begin{assumption}  [Characterization of invariant sets]\label{ass:chara_sets} For the limiting projected primal-dual dynamics, every internally chain-transitive invariant set is contained in the set of projected stationary points satisfying primal feasibility and complementarity:
\[
0\!\in\!
\hat{\Gamma}_\pi
\left(
\nabla_\pi L_{\mathrm{EAL}}(\pi,\lambda)
\right),
\quad\!\!\!\!
\phi_i(\pi)\!\leq\! 0,
\quad\!\!\!\!
\lambda_i\phi_i(\pi)\!=\!0,
\quad\!\!\!\!
\lambda_i\!\geq\! 0, \quad\!\!\!\!
 \forall i.
\]
\end{assumption}

\begin{theorem}[Convergence]\label{theorem_convergence}
Under Assumptions \ref{ass:continuity-reward}--\ref{ass:chara_sets}, the sequence $(\pi_k,\lambda_k)$ generated by PPO--EAL is almost surely bounded.
Moreover, the iterates converge almost surely to an $O(\varepsilon)$ neighborhood of an internally chain-transitive invariant set $\mathcal M$ of the limiting projected primal-dual dynamics:
\[
\limsup_{k\rightarrow\infty}
\operatorname{dist}
\left(
(\pi_k,\lambda_k),\mathcal M
\right)
\leq O(\varepsilon).
\]
If the PPO approximation bias vanishes asymptotically, i.e.,
$\varepsilon\rightarrow 0$, then the iterates converge almost surely to $\mathcal M$. Furthermore, if the limit points in $\mathcal M$ satisfy the local exactness conditions in Theorem~1, then they correspond to KKT points of the original constrained problem~\eqref{eq:primal_problem}.

\end{theorem}

\begin{proof}[Proof sketch]

The proof follows the standard multi-time-scale SA and ODE method. We outline the main ideas in the following.

\paragraph*{Step 1: SA recursion} Under Assumptions \ref{ass:continuity-reward}--\ref{ass:proj_biase}, the PPO--EAL updates can be written as the projected SA recursion
\begin{align}
\pi_{k+1}
&=
\Gamma_\pi
\left(
\pi_k
+
\alpha_k
\left(
\nabla_\pi L_{\mathrm{EAL}}(\pi_k,\lambda_k)
+
\xi_{k+1}
+
\Delta^\pi_k
\right)
\right),\nonumber
\\
\lambda_{k+1}
&=
\Gamma_\lambda
\left(
\lambda_k
+
\omega_k
\left(
\phi(\pi_{k+1})
+
\zeta_{k+1}
+
\Delta^\lambda_k
\right)
\right),\nonumber
\end{align}
where $\xi_{k+1}$ and $\zeta_{k+1}$ are martingale-difference noise terms, and $\Delta^\pi_k$ and $\Delta^\lambda_k$ denote the bounded PPO-induced perturbations. By Assumption~\ref{ass:ppo-bias},
\[
\|\Delta^\pi_k\|+\|\Delta^\lambda_k\|=O(\varepsilon).
\]
The positive sign in the policy recursion follows from the fact that the policy update maximizes the exact augmented Lagrangian objective with respect to $\pi$, whereas the multiplier update increases $\lambda_i$ when the corresponding constraint violation $\phi_i(\pi)$ is positive.

\paragraph*{Step 2: Limiting projected dynamics} Since
$\omega_k/\alpha_k\rightarrow 0$, the policy evolves on the fast time scale and the multiplier evolves on the slow time scale. Neglecting the martingale-difference noise and the bounded PPO perturbations, the limiting projected primal-dual dynamics are
\[
\dot{\pi}
\in
\hat{\Gamma}_\pi
\left(
\nabla_\pi L_{\mathrm{EAL}}(\pi,\lambda)
\right),
\qquad
\dot{\lambda}
\in
\hat{\Gamma}_\lambda
\left(
\phi(\pi)
\right).
\]
By Assumptions~\ref{ass:continuity-reward},~\ref{ass:feasiblity}, and~\ref{ass:proj_biase}, the projected vector fields are well-defined on compact sets and are locally Lipschitz. Hence, the limiting projected dynamics admit well-defined solution trajectories.

\paragraph*{Step 3: ODE tracking and invariant-set convergence} By the Robbins--Monro step-size conditions in Assumption~\ref{ass:stepsizes}, the martingale-difference noise condition in Assumption~\ref{ass:noise}, and the compact
projection in Assumption~\ref{ass:proj_biase}, the interpolated trajectory of the stochastic recursion is an asymptotic pseudo-trajectory of the limiting projected
dynamics, up to the bounded PPO perturbation $O(\varepsilon)$. Standard projected stochastic approximation results then imply that the iterates converge almost surely to an internally chain-transitive invariant set of the limiting dynamics, with an $O(\varepsilon)$ perturbation:
\[
\limsup_{k\rightarrow\infty}
\operatorname{dist}
\left(
(\pi_k,\lambda_k),\mathcal M
\right)
\leq O(\varepsilon).
\]
When $\varepsilon\rightarrow 0$, the perturbed dynamics reduce to the unperturbed limiting projected dynamics, and the iterates converge almost surely to $\mathcal M$.

\paragraph*{Step 4: Identification of KKT limit points} By Assumption 7, every internally chain-transitive invariant set of the limiting projected dynamics is contained in the set of projected stationary points satisfying primal feasibility, dual feasibility, and complementarity. Therefore, every limit point in $\mathcal M$ satisfies the first-order projected stationarity and complementarity conditions. If, in addition,
the local exactness conditions in Theorem~1 hold, then the exactness result implies that these limit points correspond to KKT points of the original constrained problem (5). This completes the proof.

\end{proof}

\subsection{Momentum Regulation for Updating Lagrangian Multiplier}
The above updating process ensures the exactness without introducing extra variables. However, the usage of the maximal operator result in non-smooth behavior, thus could cause oscillation in constraint satisfaction. To address this issue, we propose to update the Lagrange multipliers by additionally penalizing the rate of change of the cost. That is,
differing from the \textit{n{\"a}ive} primal-dual scheme in \eqref{eq:exa_pdd_update_lambda}, we update $\lambda'_{i}$ by
{\small
	\begin{align}
		\lambda'_{i} &\!=\! \max\bigl\{0,\lambda_{i} \!+\! \alpha_{\lambda_{i}}\phi_i(\pi_{\theta})\!+\! k_d(J_{C_{i}}(\pi_{\theta}) \!-\! J_{C_{i}}(\pi_{\theta})^{\text{old}}) \bigr\}, 
		\label{eq:exa_pdd_update_lambda_pd}
	\end{align}
}%
where, $J_{C_{i}}(\pi_{\theta})^{\text{old}}$ and $k_d$ are the last constraint violation ratio and derivative gain.  

\begin{remark}[Effect of momentum regularization on convergence]
We introduce a derivative-type correction term in the
dual update, resulting in
\[
\lambda_{k+1}=\Gamma_\lambda
\left(\lambda_k+\omega_k\left[\phi(\pi_{k+1})+k_d\left(J_C(\pi_{k+1})-J_C(\pi_k)\right)\right]\right).
\]
Under the Lipschitz continuity of the constraint cost $J_C(\pi)$ and the boundedness of the projected policy update, we have
\[
J_C(\pi_{k+1})-J_C(\pi_k)=O(\|\pi_{k+1}-\pi_k\|)=O(\alpha_k).
\]
Therefore, the additional derivative-type contribution to the dual recursion satisfies
\[\omega_k k_d\left(J_C(\pi_{k+1})-J_C(\pi_k)\right)=O(\omega_k\alpha_k).
\]
Under the step-size condition in Assumption~3, this term is a bounded vanishing perturbation of the dual recursion. Hence, in the stochastic-approximation framework, it does not change the limiting projected primal-dual dynamics:
\[
\dot{\pi} \in \hat{\Gamma}_\pi
\left(\nabla_\pi L_{\mathrm{EAL}}(\pi,\lambda)\right),
\qquad
\dot{\lambda} \in \hat{\Gamma}_\lambda
\left(\phi(\pi)\right).
\]
Consequently, the internally chain-transitive invariant set
characterized in Theorem~2 remains unchanged, and the convergence result continues to hold with the derivative-type dual regulation.

Intuitively, the derivative-type correction acts as a damping component in the dual dynamics. It reduces transient oscillations of the Lagrange multipliers without altering the limiting KKT-associated invariant set.
\end{remark}

\subsection{Actor-critic Implementation}
Using the actor-critic structure, a practical implementation of PPO-EAL follows Algorithm~\ref{alg:PPO_EAL}. The difference with PPO is marked in blue.
\begin{algorithm}[!t]
	\caption{PPO-EAL Outline}
	\label{alg:PPO_EAL}
	\KwIn{Initialize policy network $\pi_0$, value network $V^{\pi}_{R}$, \blue{and
			cost value networks $V^{\pi}_{C_i}$ for all cost function $c_i$}; 
		{Initialize policy parameter $\theta_0$, \blue{Lagrange multipliers $\lambda_0$, penalty factor $\beta_0$}} }
	\KwOut{Optimized policy parameters $\theta$}
	\BlankLine
	\For(\tcp*[f]{Episode loop}){$k \leftarrow 1,\ldots,K$}{
		Generate trajectories $\tau \sim \pi_{\theta_k}$
		
		Estimate returns and advantage functions
		
		\blue{\eIf{\texttt{momentum regulation == True}}{Update Lagrangian multiplier $\lambda'_{i}$ using~\eqref{eq:exa_pdd_update_lambda_pd},}{Update Lagrangian multiplier $\lambda'_{i}$ using~\eqref{eq:exa_pdd_update_lambda}.}}
		
		\For (\tcp*[f]{Epochs loop}){$e \leftarrow 1,\ldots,E$}{
			\For (\tcp*[f]{Mini-batch loop}){$b \leftarrow 1,\ldots,B$}{
				Update value network $V^{\pi}_{R}$ using~\eqref{eq:ppo_clip_reward}
				
				\blue{Update cost networks $V^{\pi}_{C_i}$ using~\eqref{eq:ppo_clip_cost}}
				
				Update policy network using \eqref{eq:pdd_update_theta}
				
			}
		}
	}
\end{algorithm}

Note that in real applications, we can first tune the hyperparameters, e.g., $\alpha_{\theta}$ and $\alpha_{\lambda_i}$ in PPO-EAL, without considering the momentum regulation. Then, we can fix these parameters and adjust $k_d$ to incorporate the momentum regulation.   

\begin{remark} Unlike the slack-augmented AL used in APPO \cite{dai2023augmented}, our PPO-EAL formulation directly adopts the Rockafellar-type exact augmented Lagrangian without introducing auxiliary slack variables. Hence, no penalty-factor adjustment is required in practice. In addition, differing from the APPO method, we update the Lagrangian $\lambda_i$ only at the outer (epoch) loop  rather than inside each epoch, ensuring smoother and more stable convergence behaviour.
\end{remark}

\section{Validation}\label{evaluation_exp}
We systematically evaluate PPO-EAL across various control tasks with increasing dynamical and operational complexity, ranging from classical inverted cart-pole stabilization problems to high-dimensional manipulation and locomotion tasks. These tasks encompass multiple physical safety constraints, including velocity, position, torque, and contact-force limitations, providing a comprehensive assessment of safe control capability. Using GPU-accelerated simulation for large-scale training and benchmarking, we compare PPO-EAL against representative state-of-the-art first-order safe RL baselines. Beyond simulation, we further investigate real-world deployability through zero-shot sim-to-real transfer in a contact-rich gear assembly task, assessing safety, robustness, and control effectiveness in physical robotic systems. All resulting videos are available at \href{https://youtu.be/Cenyd5Ng76k}{https://youtu.be/Cenyd5Ng76k}.

\subsection{Environment and Baselines}
We train the policy with IssacLab, a GPU-accelerated, open-source framework designed to unify and simplify robotics research workflows. The IssacLab development initiated from the Orbit framework \cite{mittal2023orbit} and the robot dynamics are emulated with Issac Sim \cite{NVIDIA_Isaac_Sim}. We implemented the RL algorithms with RSL-RL, an open-source RL library \cite{schwarke2025rsl}. Within each task, all methods employ identical network architectures, with ELU activation functions for fair comparison. By default, all the tasks are trained with 3000 iteration steps.

To comprehensively assess task performance, we compare PPO-EAL against representative state-of-the-art first-order constrained RL baselines, including PPO Lagrangian (PPO-L)~\cite{ray2019benchmarking}, P3O\footnote{Here, we use the normalized version defined in \cite{lee2024exploring}.}, and APPO \cite{dai2023augmented} (we call it `PPO-AL' below). In particular, we extend the above methods to accommodate multiple feasibility constraints if necessary. To solve the CMDP problem, \(c_i\) is defined as a binary threshold-violation indicator, and \(J_{C_i}\) is estimated by the empirical violation rate over all rollout samples. Therefore, the constraint \(J_{C_i}(\pi) \le b_i\) requires the average violation rate of the \(i\)-th physical constraint to remain below the prescribed threshold \(b_i\)\footnote{Note that we actually obtain an approximation version of `violation rate' by defining the discounted-sum cost $J_{C_{i}}(\pi)$ in \eqref{eq:cmdp_specify}.}. Moreover, in all validations, the standard PPO serves as a basic baseline.

\subsection{Safe Control Tasks with GPU-Accelerated Environment}
In this section, we validate and compare different safe RL methods across four benchmark control tasks, including inverted cart-pole balancing, cart-double pendulum stabilization, Franka end-effector pose reaching and quadrupedal locomotion, as illustrated in Fig.~\ref{fig:classica_learning_control-tasks}. 
\begin{figure}
	\centering
	\includegraphics[width=0.49\textwidth]{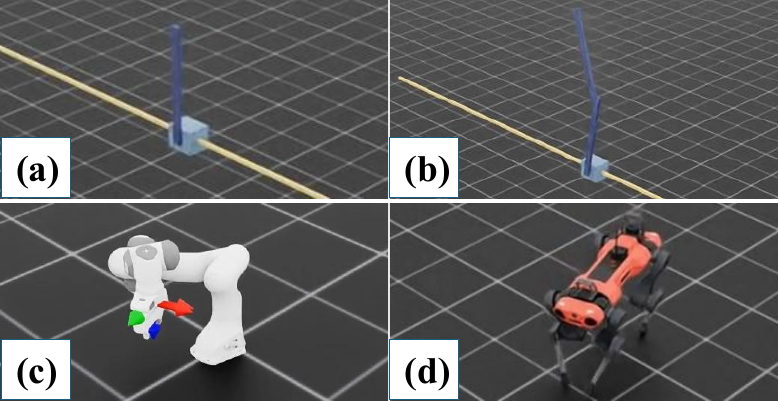}
	\caption{Four benchmark control tasks used for RL evaluation: (a) inverted cart-pole balancing, (b) cart-double pendulum swing-up and stabilization, (c) Franka end-effector pose reaching, (d) quadrupedal locomotion.}
	\label{fig:classica_learning_control-tasks}
\end{figure}

\subsubsection{Inverted cart-pole balancing} Considering an inverted pendulum attached to a moving cart, we compute the desired force applied to the cart base to ensure that the pole moves to the vertical position rapidly and then maintains balance (see Fig~\ref{fig:classica_learning_control-tasks}(a)). We limit the joint velocity below 1\,rad/s and set the constraint violation threshold, i.e., safety threshold, to 1e-2. 
We apply the actor-critic structure to learn the control policy. Both the actor and critic networks have two hidden layers, each containing 64 neurons. The agent observes cart position, cart velocity, pole inclination angle, and angular velocity and then outputs the 1D raw action applied to the cart.

\noindent\textit{Results:} Fig.~\ref{fig:cart_pole_policy} plots the training curves and Table~\ref{tab:result_cart_pole_task} reports the statistical results. Results in Fig.~\ref{fig:cart_pole_policy} and Table~\ref{tab:result_cart_pole_task} demonstrate that PPO-L, P3O, PPO-EAL, and PPO-EAL-m satisfy the safety requirement\footnote{Note that we are actually using the approximation version of `violation rate' due to the usage of discounted-sum cost $J_{C_{i}}(\pi)$, as defined in \eqref{eq:cmdp_specify}.}. While PPO-L obtains the lowest limit violation rate, PPO-EAL-m achieves the highest task reward while satisfying the safety requirement, thus accomplishing the best overall safety-performance tradeoff. 

\begin{figure*}
	\centering
	\begin{minipage}[c]{0.495\textwidth}
		\centering
		\subcaptionbox{Pole joint velocity threshold violation during the training process.\label{fig:cost1_cart_pole}}[\textwidth]{%
			\includegraphics[width=\textwidth]{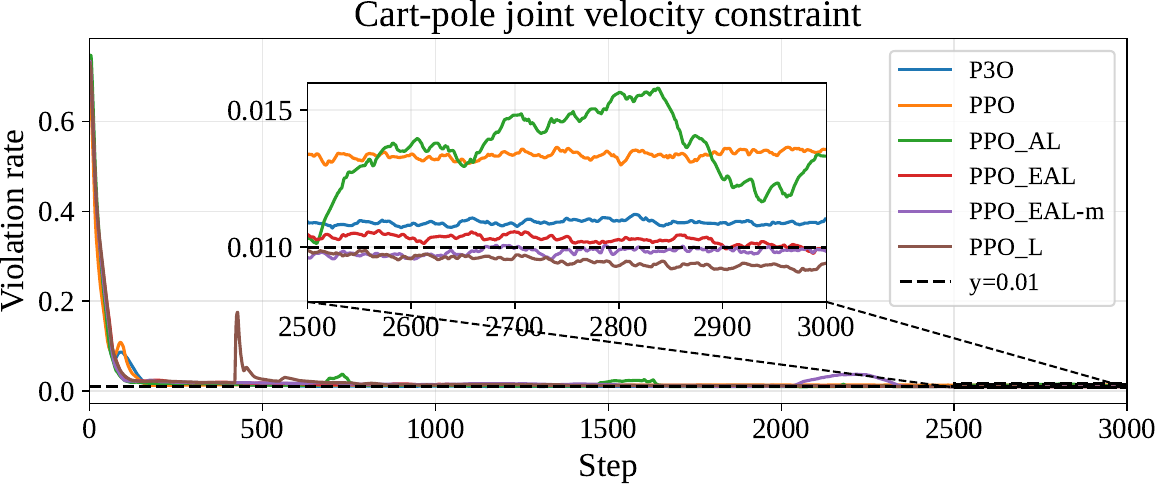}
		}
	\end{minipage}%
	\hfill
	\begin{minipage}[c]{0.495\textwidth}
		\centering
		\subcaptionbox{Evolution of the mean reward during the training process.\label{fig:cost2_cart_pole}}[\textwidth]{%
			\includegraphics[width=\textwidth]{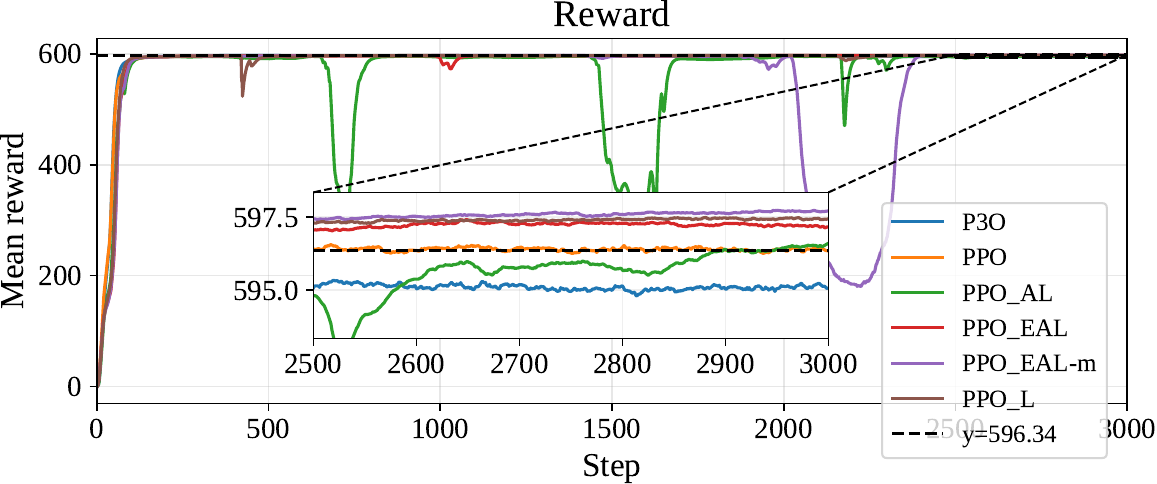}
		}
	\end{minipage}
	\caption{Training profiles for the inverted cart-pole balancing. The black dashed lines in (a) and (b) separately mark the safety threshold (a) and the reward obtained by vanilla PPO (b).}
	\label{fig:cart_pole_policy}
\end{figure*}

\begin{table}
	\centering
	\renewcommand{\arraystretch}{1.25}
	\caption{Statistical results for the inverted cart-pole balancing. Red fonts mark the safety requirement violation. Light-blue values denote minor numerical deviations attributable
	to statistical estimation noise or stochastic sampling variance\protect\footnotemark. The bold black fonts mark the peak values obtained by different policies.}
	\label{tab:result_cart_pole_task}  
	\resizebox{0.49\textwidth}{!}{
		\begin{tabular}{cccccc}
			\toprule 
			Method
			& \multicolumn{2}{c}{Violation rate ($\downarrow$)} 
			& \multicolumn{2}{c}{Reward ($\uparrow$)} & {Success}\\
			& Last 100 steps & End value & Last 100 steps & End value &{}\\
			\midrule
			\textbf{PP0}& 1.3e-2$\pm$7.1e-5 &\textcolor{red}{1.4e-2} & {596.37}$\pm$4.6e-2 & 596.34& {\ding{55}} \\ 
			\midrule
			PPO-L& 9.3e-3$\pm$8.7e-4 & \textbf{9.4e-3} & {597.43}$\pm$2.2e-2 & 597.43& {\ding{51}} \\  
			P30& 1.1e-2$\pm$5.0e-5 & \cyan{1.1e-2} & {595.09}$\pm$5.7e-2 & 595.06& {\ding{51}\ding{55}}\\ 
			PP0-AL& 2.6e-2$\pm$7.8e-3 & \textcolor{red}{3.4e-2} & {596.41}$\pm$8.2e-2 & 596.59& {\ding{55}} \\ 
			PP0-EAL& 1.0e-2$\pm$9.2e-5 & 9.9e-3 & {597.22}$\pm$3.8e-2 & 597.15& {\ding{51}} \\ 
			PP0-EAL-m& 9.9e-3$\pm$6.8e-5 & 9.9e-3 & {597.68}$\pm$2.8e-2 & \textbf{597.70}& {\ding{51}} \\ 
			\bottomrule
		\end{tabular}
	}
\end{table}
\footnotetext{These marginal deviations are not considered  safety violations in our
	evaluation. We mark this case with `\ding{51}\ding{55}' in the table.}

\subsubsection{Cart-double pendulum stabilization}
Taking into account the double inverted pendulums on a moving cart, we compute the desired force on the cart base and the rotation torque applied to the second pendulum joint to ensure that the double pendulums move to the vertical position rapidly (see Fig~\ref{fig:classica_learning_control-tasks}(b)). Specifically, cart position and velocity are constrained to 1.5\,m and 3\,m/s, respectively. The constraint violation thresholds are 5e-3 and 1e-2, respectively. 
The network structure is the same as that used for the cart-pole balancing task. The agent observes the cart position, cart linear velocity, pole angle, pole angular velocity, relative angle of pendulum to first pole, angular velocity of the second joint, and absolute angle of second pole relative to cart vertical, and then outputs the 2D raw action applied to the system.   

\noindent\textit{Results:} Fig.~\ref{fig:cart_double_pendulum_policy} and Fig.~\ref{fig:cart_double_pendulum_reward} separately plot the constraint violation curves and the reward curves with different methods and Table~\ref{tab:result_cart_double_pendulum_task} reports the statistical results. The results demonstrate that all first-order constrained RL approaches evaluated satisfy the safety requirement. However, PPO-EAL-m achieves the highest task reward while preserving precise safety compliance. 

\begin{figure*}
	\centering
	
	\begin{minipage}[c]{0.49\textwidth}
		\centering
		\subcaptionbox{Cart position threshold violation during the training process.\label{fig:cost1_cart_double_pendulum}}[\textwidth]{%
			\includegraphics[width=\textwidth]{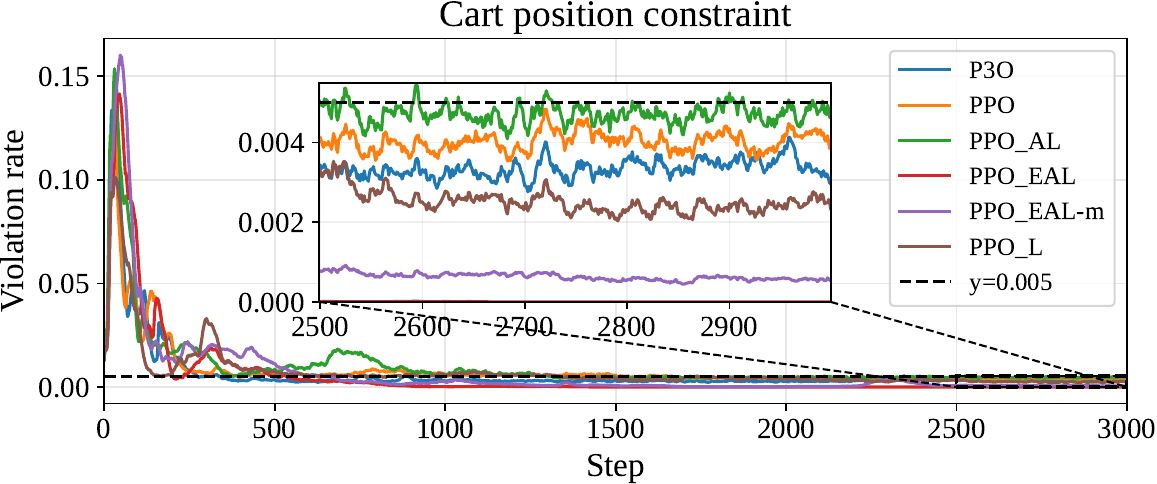}
		}
	\end{minipage}%
	\hfill
	\begin{minipage}[c]{0.49\textwidth}
		\centering
		\subcaptionbox{Cart velocity threshold violation during the training process.\label{fig:cost2_cart_double_pendulum}}[\textwidth]{%
			\includegraphics[width=\textwidth]{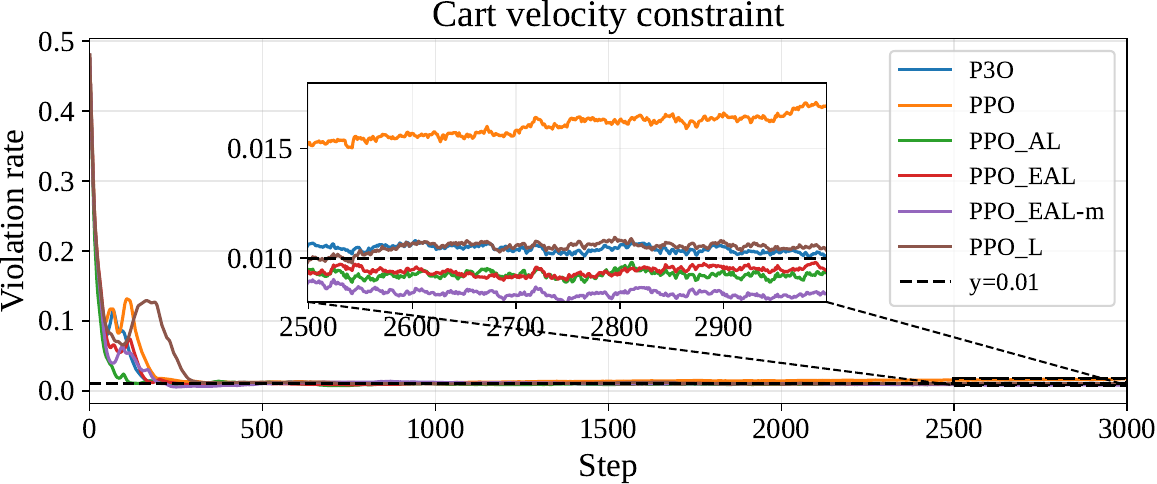}
		}
	\end{minipage}
	\caption{Training profiles for the cart-double pendulum stabilization task. The black dashed lines mark the safety thresholds.} 
	\label{fig:cart_double_pendulum_policy}
\end{figure*}

\begin{figure}
	\centering
	\includegraphics[width=0.49\textwidth]{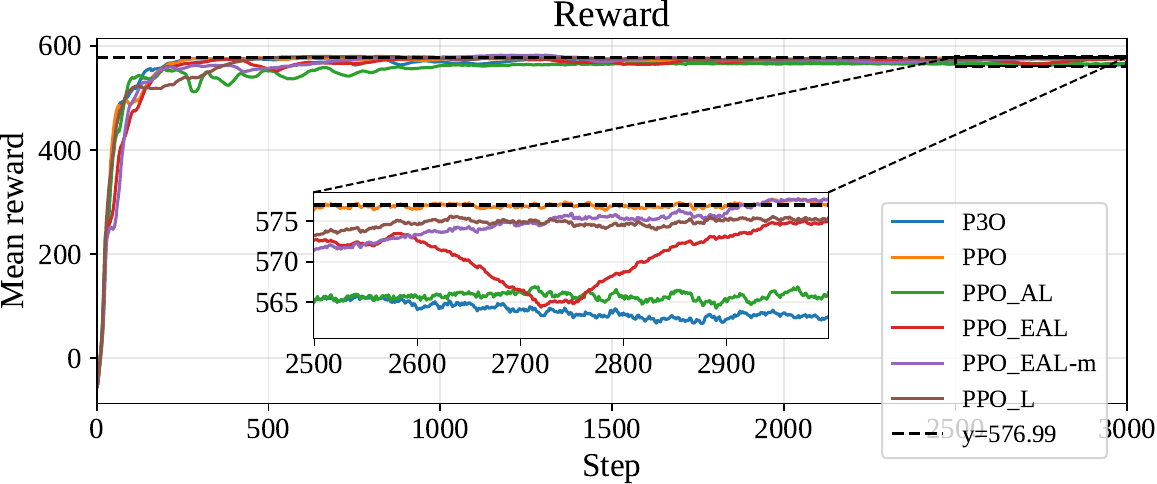}
	\caption{Reward profile for the cart-double pendulum stabilization task. The black dashed line marks the reward obtained by vanilla PPO.}
	\label{fig:cart_double_pendulum_reward}
\end{figure}

\begin{table*}[h!]
	\centering
	\renewcommand{\arraystretch}{1.25}
	\caption{Statistical results for the cart-double pendulum stabilization task. Red fonts mark the safety requirement violation. The bold black fonts mark the peak values obtained by different policies.}
	\label{tab:result_cart_double_pendulum_task}  
	\resizebox{0.75\textwidth}{!}{
		\begin{tabular}{cccccccc}
			\toprule
			Method 
			& \multicolumn{2}{c}{Position violation rate ($\downarrow$)} 
			& \multicolumn{2}{c}{Velocity  violation rate ($\downarrow$)} 
			& \multicolumn{2}{c}{{Reward ($\uparrow$)}}& {Success} \\
			& Last 100 steps & End value & Last 100 steps & End value& Last 100 steps & End value &{} \\
			\midrule
			\textbf{PP0}& 4.1e-3$\pm$2.3e-4 &3.9e-3 & 1.7e-2$\pm$2.6e-4 & \textcolor{red}{1.7e-2}& 577.26$\pm$0.28 & 576.99& {\ding{55}} \\ 
			\midrule
			PPO-L& 2.4e-3$\pm$1.4e-4 & 2.5e-3 & 1.1e-2$\pm$9.3e-5 & 1.0e-2& 575.26$\pm$0.12 & 575.18& {\ding{51}}\\    
			P30& 3.5e-3$\pm$2.4e-4 & {3.0e-3} & {1.0e-2}$\pm$9.7e-5 & 1.0e-2& 563.38$\pm$0.27 & 563.06& {\ding{51}} \\ 
			PP0-AL& 4.8e-3$\pm$2.0e-4 & 4.6e-3 & 9.2e-3$\pm$1.2e-4 & 9.3e-3& 565.77$\pm$0.47 & 565.81& {\ding{51}} \\ 
			PP0-EAL& 0.2e-23$\pm$0.0e-23 & \textbf{0.0e-23} & 9.5e-3$\pm$1.1e-4 & 9.5e-3& 576.30$\pm$0.62 & 574.92& {\ding{51}} \\ 
			PP0-EAL-m& 5.7e-4$\pm$2.8e-5 & 5.5e-4 & 8.3e-3$\pm$8.5e-5 & \textbf{8.4e-3}& 577.28$\pm$0.45 & \textbf{577.65}& {\ding{51}} \\  
			\bottomrule
		\end{tabular}
	}
\end{table*}

\subsubsection{Franka end-effector pose reaching}


Considering Franka robotic manipulator~\cite{haddadin2024franka} with a fixed base, we compute the desired joint angle (which is then tracked by the PD-type torque control) to command the end effector to track a randomly-selected position and pose (see Fig~\ref{fig:classica_learning_control-tasks}(c)). We limit the joint velocity and the end-effector Catesian velocity to 90\% of the feasible allowable thresholds. The constraint violation thresholds are both 1e-3. 
The network structure is the same as that used for the cart-pole balancing task. The agent observes the joint configuration (angles, velocities), end-effector position and velocity, goal pose (position and orientation), and last action, and then outputs the 7D raw action applied to the system, which is smoothed by a moving average filter. 

\noindent\textit{Results:} Fig.~\ref{fig:franka_policy} and Fig.~\ref{fig:franka_reward} separately plot the constraint violation curves and the reward curves. Table~\ref{tab:result_franka_task} reports the statistical results. The results demonstrate that all the first-order constrained RL approaches satisfy the safety requirement. However, due to the fluctuation, the PPO-EAL fails to meet the safety requirement on end-effector velocity strictly (see red curve in Fig.~\ref{fig:franka_policy}(b)). In contrast, when momentum regulation is considered, the safety requirement is satisfied. While P3O achieves the lowest violation rate in each channel, it results in a reduced task reward, indicating overly conservative behavior. Instead, the proposed approach, i.e., PPO-EAL and PPO-EAL-m, obtains the highest rewards, meaning the best tracking-safety performance. 

\begin{figure*}
	\centering
	
	\begin{minipage}[c]{0.49\textwidth}
		\centering
		\subcaptionbox{Joint velocity threshold violation during the training process.\label{fig:cost1_franka}}[\textwidth]{%
			\includegraphics[width=\textwidth]{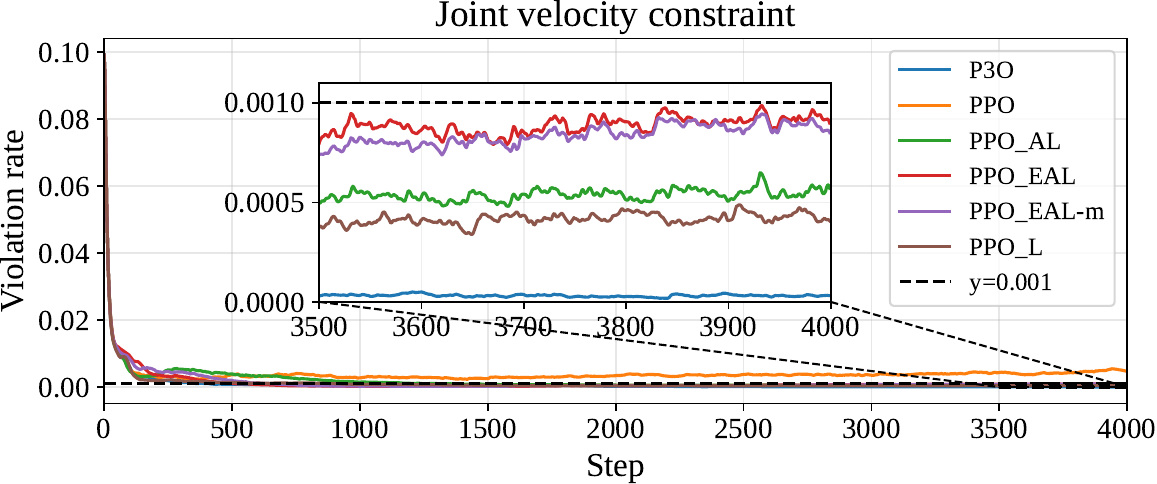}
		}
	\end{minipage}%
	\hfill
	\begin{minipage}[c]{0.49\textwidth}
		\centering
		\subcaptionbox{End-effector velocity threshold violation during the training process.\label{fig:cost2_franka}}[\textwidth]{%
			\includegraphics[width=\textwidth]{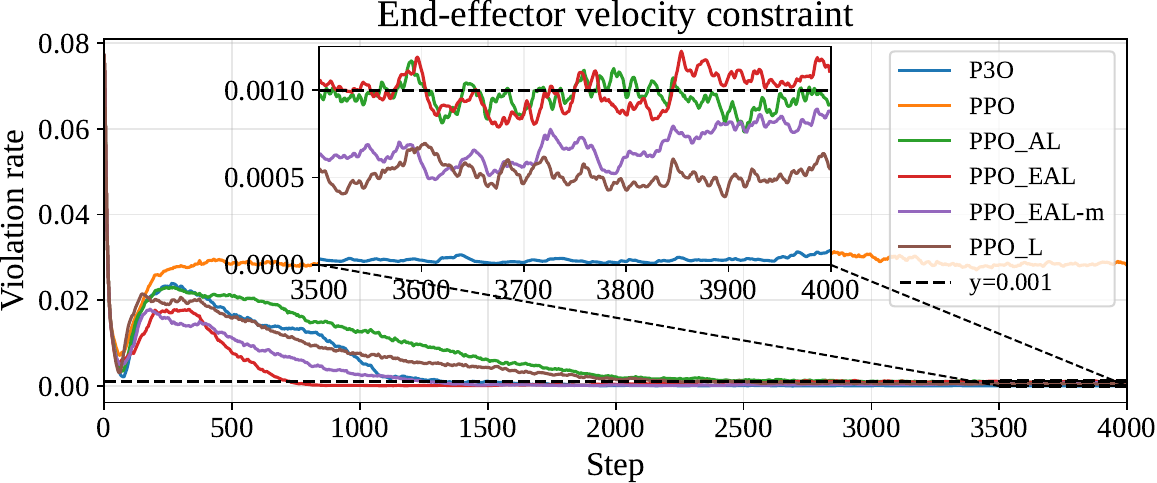}
		}
	\end{minipage}
	\caption{Training profiles for the Franka end-effector pose reaching task. The black dashed lines mark the safety thresholds.}
	\label{fig:franka_policy}
\end{figure*}

\begin{figure}
	\centering
	\includegraphics[width=0.49\textwidth]{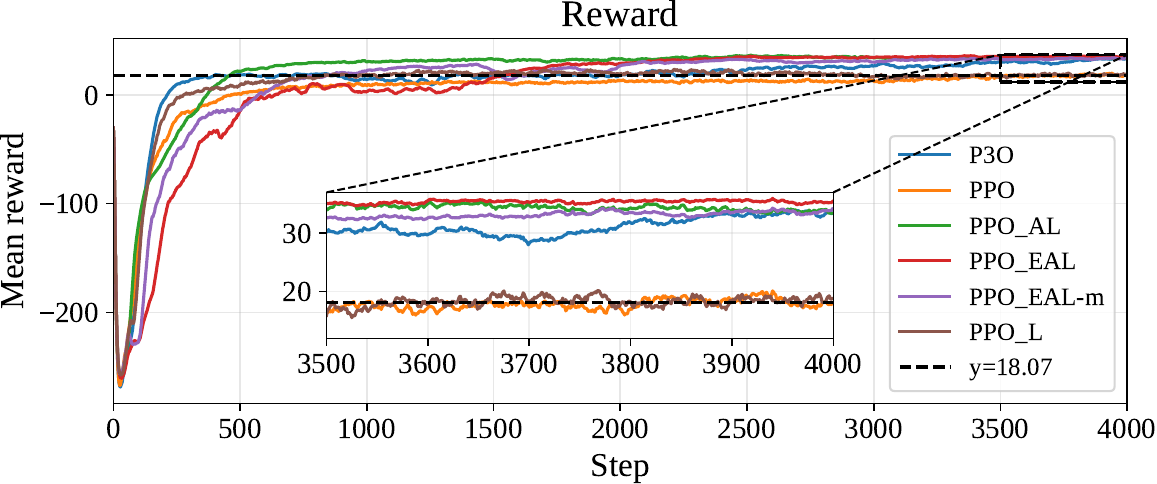}
	\caption{Reward profile for the Franka end-effector pose reaching task. The black dashed line marks the reward obtained by vanilla PPO.}
	\label{fig:franka_reward}
\end{figure}

\begin{table*}[h!]
	\centering
	\renewcommand{\arraystretch}{1.25}
	\caption{Statistical results for the Franka end-effector pose reaching task. Red fonts mark the safety requirement violation. The bold black fonts mark the peak values obtained by different policies.}
	\label{tab:result_franka_task}  
	\resizebox{0.75\textwidth}{!}{
		\begin{tabular}{cccccccc}
			\toprule
			Method
			& \multicolumn{2}{c}{Joint velocity violation rate ($\downarrow$)} 
			& \multicolumn{2}{c}{End-effector velocity violation rate ($\downarrow$)} 
			& \multicolumn{2}{c}{{Reward ($\uparrow$)}}&{Success}  \\
			& Last 100 steps & End value & Last 100 steps & End value& Last 100 steps & End value&{} \\
			\midrule
			\textbf{PP0}& 5.1e-3$\pm$2.2e-4 &\textcolor{red}{4.7e-3} & 2.9e-2$\pm$1.7e-4 & \textcolor{red}{2.8e-2}& 18.47$\pm$0.83 & 18.07& {\ding{55}} \\ 
			\midrule
			PPO-L& 4.4e-4$\pm$2.4e-5 & 4.1e-4 & 5.2e-4$\pm$4.5e-5 & 5.5e-4& 18.66$\pm$0.41 & 18.78& {\ding{51}}\\   
			P30& 3.1e-5$\pm$3.0e-6 & \textbf{3.2e-5} & 4.4e-5$\pm$1.7e-5 & \textbf{8.1e-5}& 33.29$\pm$0.32 & {33.39}& {\ding{51}} \\ 
			%
			PP0-AL& 5.6e-4$\pm$3.0e-5 & 5.7e-4 & 9.2e-4$\pm$6.1e-5 & 9.2e-4& 33.76$\pm$0.22 & 33.41 & {\ding{51}}\\ 
			PP0-EAL& 9.2e-4$\pm$2.5e-5 & {9.0e-4} & 1.1e-4$\pm$4.4e-5 & \cyan{1.1e-3}& 35.38$\pm$0.22 &\textbf{35.39} & {\ding{51}\ding{55}}\\ 
			PP0-EAL-m& 8.9e-4$\pm$2.6e-5 & 8.5e-4 & 8.2e-4$\pm$3.2e-5 & {8.8e-4}& 33.69$\pm$0.25 & {34.12} & {\ding{51}}\\ 
			\bottomrule
		\end{tabular}
	}
\end{table*}

\subsubsection{Quadrupedal locomotion}
%
We applied the RL policy to a quadrupedal locomotion task with the ANYmal-C robot, following the setting in \cite{lee2024exploring} (see Fig~\ref{fig:classica_learning_control-tasks}(d)). Both the actor and critic networks have three hidden layers, each containing 128 neurons. In particular, we constrained the joint deviation, joint velocity, and joint torque below  0.5\,rad, 6\,rad/s, and 75\,Nm, whose physical-infeasible values are 0.6\,rad, 8\,rad/s and 80\,Nm, respectively. The constraint violation thresholds are 1e-2, 1e-3 and 5e-5, respectively. 

\noindent\textit{Results:} Fig.~\ref{fig:anymal_c_policy} plot the training curves and Table~\ref{tab:result_loco_task} reports the statistical results. The results in Fig.~\ref{fig:anymal_c_policy} and Table~\ref{tab:result_loco_task} demonstrate that only PPO-L, PPO-EAL, and PPO-EAL-m satisfy the safety requirement. While P3O achieves a lower violation rate for position constraints, it does so at the cost of a substantially reduced task reward. In contrast, PPO-EAL-m achieves the highest overall reward while simultaneously satisfying all required safety thresholds, demonstrating superior safety-performance trade-offs in complex multi-constraint robotic control. These results highlight the practical effectiveness of PPO-EAL in maintaining stable, high-performance locomotion while enforcing precise safety regulation.

\begin{figure*}
	\centering
	
	\begin{minipage}[c]{0.495\textwidth}
		\centering
		\subcaptionbox{Joint position threshold violation during the training process.\label{fig:cost1_quad}}[\textwidth]{%
			\includegraphics[width=\textwidth]{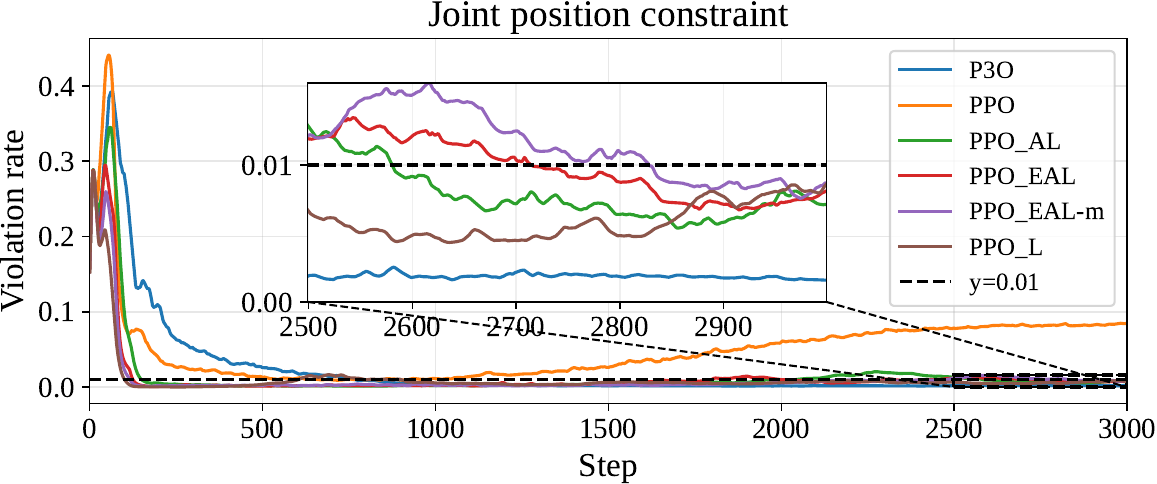}
		}
	\end{minipage}%
	\hfill
	\begin{minipage}[c]{0.495\textwidth}
		\centering
		\subcaptionbox{Joint velocity threshold violation during the training process.\label{fig:cost2_quad}}[\textwidth]{%
			\includegraphics[width=\textwidth]{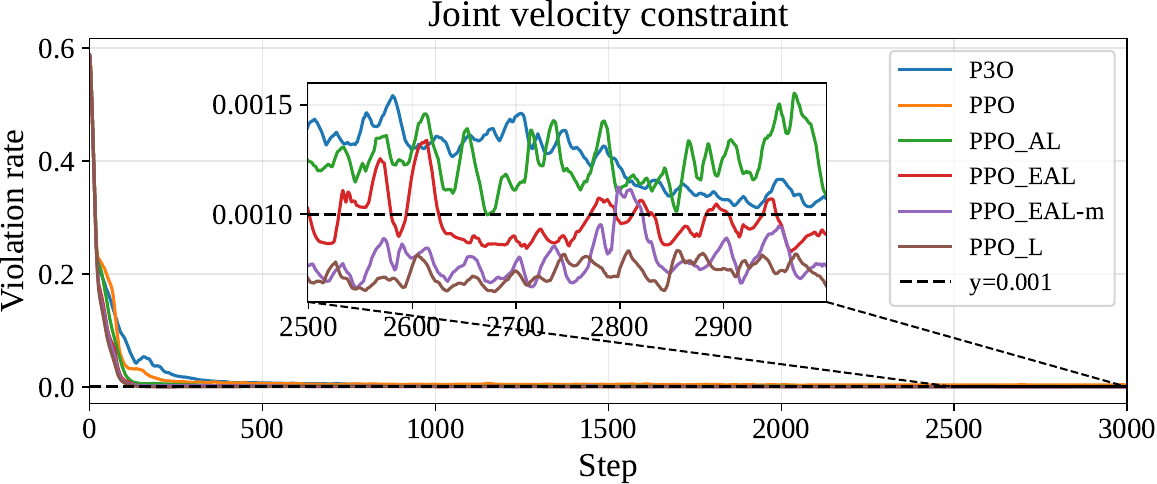}
		}
	\end{minipage}\\[1em]
	
	\begin{minipage}[c]{0.495\textwidth}
		\centering
		\subcaptionbox{Joint torque threshold violation during the training process.\label{fig:cost3_quad}}[\textwidth]{%
			\includegraphics[width=\textwidth]{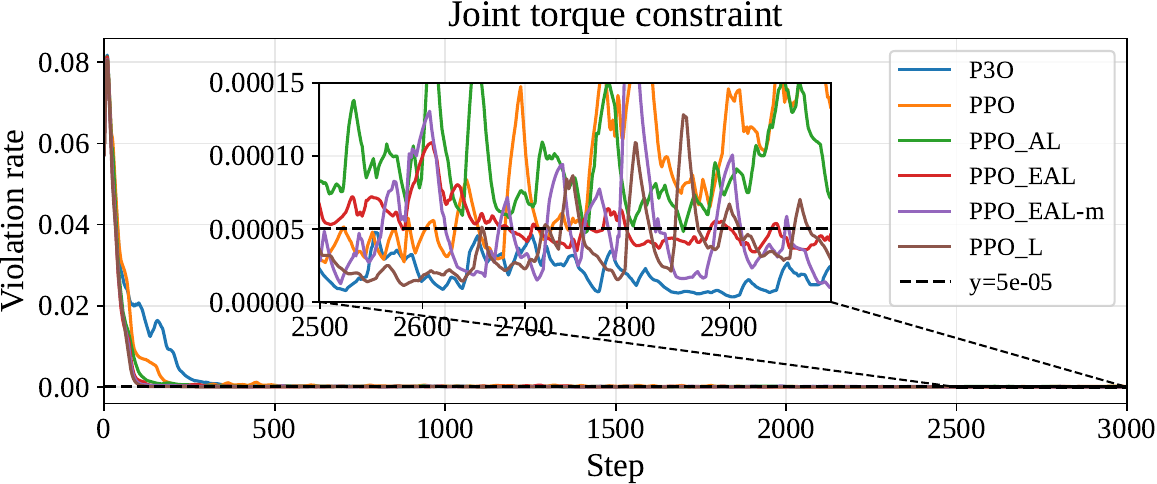}
		}
	\end{minipage}%
	\hfill
	\begin{minipage}[c]{0.495\textwidth}
		\centering
		\subcaptionbox{Evaluation of the mean reward during the training process.\label{fig:reward_quad}}[\textwidth]{%
			\includegraphics[width=\textwidth]{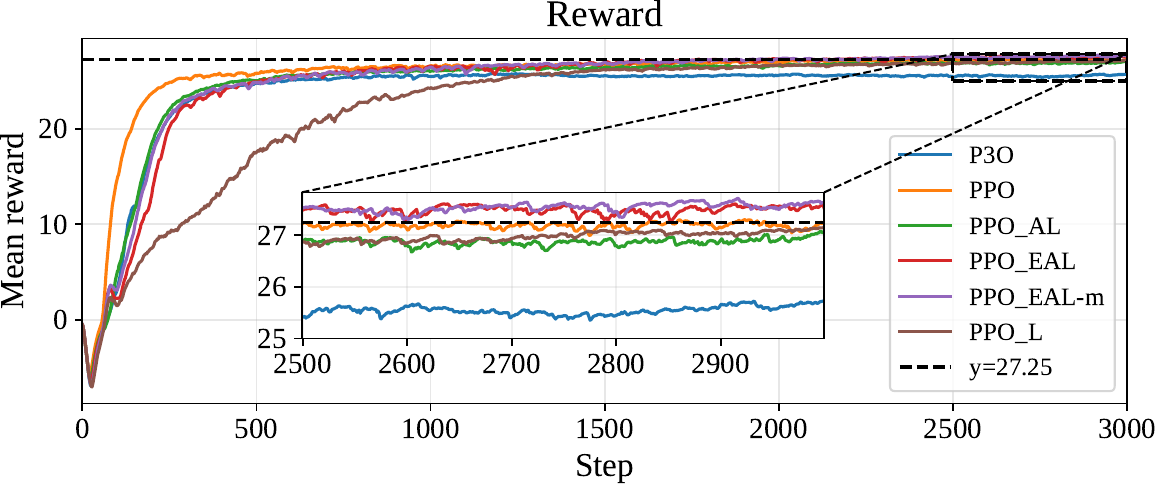}
		}
	\end{minipage}
	
	\caption{Training profiles for the quadrupedal locomotion task. The black dashed lines mark safety thresholds in (a)--(c) and the reward obtained by vanilla PPO in (d).}
	\label{fig:anymal_c_policy}
\end{figure*}

\begin{table*}[h!]
	\centering
	\renewcommand{\arraystretch}{1.25}
	\caption{Statistical results for the quadrupedal locomotion task. Red fonts mark the safety requirement violation. The bold black fonts mark the peak values obtained by different policies.}
	\label{tab:result_loco_task}  
	\resizebox{0.98\textwidth}{!}{
		\begin{tabular}{cccccccccc}
			\toprule
			Method 
			& \multicolumn{2}{c}{Position violation rate ($\downarrow$)} 
			& \multicolumn{2}{c}{Velocity violation rate ($\downarrow$)} 
			& \multicolumn{2}{c}{Torque violation rate ($\downarrow$)} 
			& \multicolumn{2}{c}{{Reward ($\uparrow$)}}&{Success} \\
			& Last 100 steps & End value & Last 100 steps & End value& Last 100 steps & End value& Last 100 steps & End value&{} \\
			\midrule
			\textbf{PP0}& {8.3e-2}$\pm$8.4e-4 &\textcolor{red}{8.5e-2} & 4.0e-3$\pm$1.9e-4 & \textcolor{red}{4.1e-3}& 1.4e-4$\pm$2.7e-5 & \textcolor{red}{1.3e-4}& 27.19$\pm$6.4e-2 & 27.25&{\ding{55}} \\ 
			\midrule
			PPO-L& 7.8e-3$\pm$4.7e-4 & 8.6e-3 & 7.6e-4$\pm$3.2e-5 & \textbf{6.7e-4}& 4.3e-5$\pm$1.1e-5 & 2.9e-5& 27.07$\pm$4.1e-2 & 27.16&{\ding{51}} \\     
			P30& 1.7e-3$\pm$8.0e-5 & \textbf{1.6e-3} & {1.1e-3}$\pm$3.7e-5 & \cyan{1.1e-3}& 1.3e-5$\pm$6.0e-6 & 2.4e-5& 25.65$\pm$4.5e-2 & 25.73&{\ding{51}\ding{55}} \\ 
			PP0-AL& 7.1e-3$\pm$6.9e-4 & 7.1e-3 & 1.3e-3$\pm$1.2e-4 & \cyan{1.1e-3}& 1.1e-4$\pm$2.4e-5 & \textcolor{red}{7.1e-5}& 26.93$\pm$5.9e-2 & 27.06 &{\ding{55}}\\ 
			PP0-EAL& 7.2e-3$\pm$3.4e-4 & 8.1e-3 & 9.4e-4$\pm$6.1e-5 & 9.1e-4& 4.2e-5$\pm$5.0e-6 & 4.6e-5& 27.54$\pm$2.8e-2 & 27.58&{\ding{51}} \\ 
			PP0-EAL-m& 8.4e-3$\pm$3.9e-4 & 8.7e-3 & 8.1e-4$\pm$7.5e-5 & 7.6e-4& 3.6e-5$\pm$2.0e-5 & \textbf{1.0e-5}& 27.59$\pm$5.5e-2 & \textbf{27.63}&{\ding{51}} \\ 
			\bottomrule
		\end{tabular}
	}
\end{table*}


\subsection{Sensitivity analysis}
Taking the quadrupedal locomotion task as an example, we conduct sensitivity analyses by varying key training settings. 
\subsubsection{Performance variation with different quadratic penalty factors}
In theory, Theorem~1 requires a sufficiently large penalty coefficient $\beta$ to guarantee the exactness property. 
However, in the above tasks, we set $\beta = \alpha_{\lambda_i}$, a constant coupled to the Lagrange learning rate, which acts as a mild quadratic regularizer stabilizing the dual updates.
This practical choice yields smoother training dynamics while maintaining {approximate exactness}.
To further investigate the influence of the penalty magnitude, we conduct a sensitivity analysis
by setting $\beta = N\,\alpha_{\lambda_i}$ with $N \in \{1,\,10,\,20,\,50,\,100\}$.

The statistical results in Table~\ref{tab:result_loco_task_sensitivity_beta} indicate that, when $N \leq 20$, the generated motion satisfies the safety requirements. In contrast, when $N$ becomes too large, some of the safety requirements are violated. 
It can be attributed to the fact that an excessively large quadratic penalty can dominate the policy-gradient update, amplify estimation noise and clipping bias, and induce oscillatory constraint correction. Another reason is that the derivative gain $k_d$ is kept fixed in this study. Since the effective dual dynamics depend jointly on $\beta$, the multiplier learning rate, and $k_d$, larger penalty values may require retuning $k_d$ for stable constraint regulation.
\begin{table*}[h!]
	\centering
	\renewcommand{\arraystretch}{1.25}
	\caption{Policy performance with different  quadratic penalty factors ($\beta = N \alpha_{\lambda_i}$). Red fonts mark the safety
		requirement violation.}
	\label{tab:result_loco_task_sensitivity_beta}  
	\resizebox{0.98\textwidth}{!}{
		\begin{tabular}{cccccccccc}
			\toprule
			Penalty
			& \multicolumn{2}{c}{Position violation rate ($\downarrow$)} 
			& \multicolumn{2}{c}{Velocity violation rate ($\downarrow$)} 
			& \multicolumn{2}{c}{Torque violation rate ($\downarrow$)} 
			& \multicolumn{2}{c}{{Reward ($\uparrow$)}}&{Success} \\
			& Last 100 steps & End value & Last 100 steps & End value& Last 100 steps & End value& Last 100 steps & End value &{} \\
			\midrule
			$N=1$ (Default)& 7.2e-3$\pm$3.4e-4 & 8.1e-3 & 9.4e-4$\pm$6.1e-5 & 9.1e-4& 4.2e-5$\pm$5.0e-6 & 4.6e-5& 27.54$\pm$2.8e-2 & 27.58&{\ding{51}} \\ 
			\midrule
			$N=10$& 8.7e-3$\pm$5.8e-4 & 8.2e-3 &8.6e-4$\pm$3.5e-5 & 8.1e-4& 6.8e-5$\pm$1.4e-5 & 4.8e-5& 27.58$\pm$3.6e-2 & 27.61&{\ding{51}} \\ 
			$N=20$& 1.2e-2$\pm$1.1e-3 & 9.8e-3 & 9.7e-4$\pm$4.6e-5 & 9.0e-4& 3.6e-5$\pm$1.3e-5 & 3.0e-5& 27.45$\pm$3.7e-2 & 27.52&{\ding{51}} \\ 
			$N=50$& 7.3e-3$\pm$9.1e-4 & 6.7e-3 & 7.4e-4$\pm$4.4e-5 & 7.0e-4& 8.0e-5$\pm$1.0e-5 & \red{7.7e-5}& 27.36$\pm$3.0e-2 & 27.39&{\ding{55}} \\ 
			$N=100$& 1.1e-2$\pm$3.7e-4 & \red{1.2e-2} & 8.5e-4$\pm$6.8e-5 & 7.6e-4& 4.1e-5$\pm$2.5e-5 & 1.7e-5& 27.47$\pm$5.0e-2 & 27.49&{\ding{55}} \\     
			\bottomrule
		\end{tabular}
	}
\end{table*}

\subsubsection{Robustness against different random seeds}
To evaluate robustness under environmental uncertainty, we assess PPO-EAL across multiple runs using five different random seeds. Table~\ref{tab:result_loco_task_sensitivity_random_seed} summarizes the resulting policy performance under varying initialization and environment stochasticity. 

The upper part of Table~\ref{tab:result_loco_task_sensitivity_random_seed} reports PPO-EAL without momentum regulation. While the method generally maintains strong reward performance across seeds, the resulting policies exhibit noticeable sensitivity in satisfying strict torque safety constraints, with several runs violating the required torque threshold. Although the average torque violation remains relatively low, these results indicate that the baseline dual-update mechanism can produce unstable constraint regulation under environmental variation.

The lower panel reports PPO-EAL with momentum regulation, i.e, PPO-EAL-m. Incorporating momentum significantly improves robustness by reducing overall torque violation, improving average safety satisfaction, and increasing consistency across random seeds. In particular, the average policy with momentum regulation satisfies all required safety constraints while preserving comparable or slightly improved reward performance relative to the non-momentum variant. These results demonstrate that momentum regulation not only improves transient safety-performance trade-offs, but also enhances training robustness under varying environmental uncertainties.
\begin{table*}[h!]
	\centering
	\renewcommand{\arraystretch}{1.25}
	\caption{Policy performance with different random seeds. The upper panel reports the results without momentum regulation while the bottom panel reports the results with momentum regulation.}
	\label{tab:result_loco_task_sensitivity_random_seed}  
	\resizebox{0.98\textwidth}{!}{
		\begin{tabular}{cccccccccc}
			\toprule
			Seed
			& \multicolumn{2}{c}{Position violation rate ($\downarrow$)} 
			& \multicolumn{2}{c}{Velocity violation rate ($\downarrow$)} 
			& \multicolumn{2}{c}{Torque violation rate ($\downarrow$)} 
			& \multicolumn{2}{c}{{Reward ($\uparrow$)}}&{Success} \\
			& Last 100 steps & End value & Last 100 steps & End value& Last 100 steps & End value& Last 100 steps & End value &{}\\
			\midrule
			42 (Default)& 7.2e-3$\pm$3.4e-4 & 8.1e-3 & 9.4e-4$\pm$6.1e-5 & 9.1e-4& 4.2e-5$\pm$5.0e-6 & 4.6e-5& 27.54$\pm$2.8e-2 & 27.58 &{\ding{51}}\\  
			0& 1.2e-3$\pm$1.2e-3 & \cyan{1.1e-2} & 1.3e-3$\pm$1.1e-4 & \red{1.2e-3}& 8.1e-5$\pm$3.6e-5 & \red{8.0e-5}& 27.30$\pm$4.2e-2 & {27.31}&{\ding{55}} \\  
			1& 9.6e-3$\pm$5.3e-4 & \cyan{1.1e-2} & 9.8e-4$\pm$7.6e-5 & 9.2e-4& 6.9e-5$\pm$2.6e-5 & {4.3e-5}& 27.36$\pm$4.6e-2 & {27.46} &{\ding{51}\ding{55}}\\  
			2& 7.9e-3$\pm$2.2e-4 & 7.8e-3 & 1.0e-3$\pm$4.2e-5 & {9.4e-4}& 1.0e-4$\pm$1.5e-5 & \red{9.5e-5}& 27.47$\pm$5.7e-2 & {27.53}&{\ding{55}} \\      
			199& 8.5e-3$\pm$4.5e-4 & 8.1e-3 & 1.0e-3$\pm$6.9e-5 & 9.8e-3& 1.3e-4$\pm$2.2e-5 & \red{1.2e-4}& 27.12$\pm$7.4e-2 & {27.18}&{\ding{55}} \\
			\textbf{Average} &
			8.9{e-3}$\pm$1.7{e-3} &
			9.1{e-3}$\pm$1.5{e-3} &
			1.1{e-3}$\pm$1.6{e-4} &
			1.0{e-3}$\pm$1.3{e-4} &
			8.4{e-5}$\pm$3.7{e-5} &
			\red{7.7{e-5}}$\pm$3.3{e-5} &
			27.36$\pm$0.15 &
			27.39$\pm$0.15 &{\ding{55}}\\         
			\midrule
			42 (Default)& 8.4e-3$\pm$3.9e-4 & 8.7e-3 & 8.1e-4$\pm$7.5e-5 & 7.6e-4& 3.6e-5$\pm$2.0e-5 & {1.0e-5}& 27.59$\pm$5.5e-2 & {27.63}&{\ding{51}} \\  
			0& 7.9e-3$\pm$5.8e-4 & 8.6e-3 & 9.7e-4$\pm$7.5e-5 & 9.7e-4& 6.9e-5$\pm$1.3e-5 & \red{6.6e-5}& 27.54$\pm$4.7e-2 & {27.59} &{\ding{55}}\\  
			1& 1.2e-2$\pm$5.6e-4 & \cyan{1.1e-2} & 7.1e-4$\pm$3.4e-5 & 7.0e-4& 3.6e-5$\pm$1.3e-5 & {1.7e-5}& 27.67$\pm$3.7e-2 & {27.70}&{\ding{51}\ding{55}} \\  
			2& 9.3e-3$\pm$3.8e-4 & 9.5e-3 & 1.3e-3$\pm$8.6e-5 & \red{1.2e-3}& 5.9e-5$\pm$2.8e-5 & \red{6.5e-5}& 27.45$\pm$5.8e-2 & {27.47} &{\ding{55}}\\      
			199& 6.8e-3$\pm$4.2e-4 & 8.1e-3 & 7.3e-4$\pm$4.2e-5 & 7.5e-4& 1.0e-5$\pm$7.0e-6 & {1.0e-5}& 27.37$\pm$4.2e-2 & {27.46}&{\ding{51}} \\ 
			\textbf{Average} &8.9{e-3}$\pm$1.8{e-3} &9.1{e-3}$\pm$1.0{e-3} &9.0{e-4}$\pm$2.2{e-4} &8.8{e-4}$\pm$2.3{e-4} &4.2{e-5}$\pm$2.7{e-5} &3.4{e-5}$\pm$2.9{e-5} &27.52$\pm$0.11 &27.57$\pm$0.10 &{\ding{51}}\\    
			\bottomrule    
		\end{tabular}
	}
\end{table*}

\subsubsection{Task performance under different safety requirements}
To evaluate sensitivity to varying safety specifications, we assess PPO-EAL and PPO-EAL-m under multiple combinations of constraint thresholds while keeping all training hyperparameters fixed. This setting tests the versatility of the proposed framework under changing operational safety demands without task-specific retuning.

Table~\ref{tab:result_loco_task_sensitivity_threstholod} demonstrates that extremely restrictive safety thresholds (e.g., 1e-5 for torque limit violation) substantially reduce the feasible solution space for both methods, occasionally leading to constraint violations. Nevertheless, PPO-EAL-m achieves superior safety robustness and more stable safety-performance trade-offs across most evaluated configurations. These results demonstrate that the momentum-regulated PPO-EAL framework offers improved versatility and practical applicability for robotic systems operating under different levels of safety requirements.

\begin{table*}[h!]
	\centering
	\renewcommand{\arraystretch}{1.25}
	\caption{Policy performance with different safety requirements. Red fonts mark the safety
		requirement violation.}
	\label{tab:result_loco_task_sensitivity_threstholod}
	\resizebox{0.99\textwidth}{!}{%
		\begin{tabular}{ccccccccccc}
			\toprule
			{Safety} & Method
			& \multicolumn{2}{c}{Position violation rate ($\downarrow$)}
			& \multicolumn{2}{c}{Velocity violation rate ($\downarrow$)}
			& \multicolumn{2}{c}{Torque violation rate ($\downarrow$)}
			& \multicolumn{2}{c}{Reward ($\uparrow$)}
			& {Success} \\
			level & & Last 100 steps & End value & Last 100 steps & End value & Last 100 steps & End value & Last 100 steps & End value & \\
			\midrule
			\multirow{2}{*}{[1e-2,1e-3,1e-4]} & PPO-EAL
			& 1.1e-2$\pm$9.4e-4 & \red{1.4e-2}
			& 9.1e-4$\pm$4.1e-5 & 9.3e-4
			& 2.2e-5$\pm$8.0e-6 & 1.5e-5
			& 27.56$\pm$5.2e-2 & 27.56
			& {\ding{55}} \\
			{}& PPO-EAL-m
			& 1.2e-2$\pm$5.6e-4 & \cyan{1.1e-2} 
			& 9.2e-4$\pm$8.7e-5 & 9.4e-4
			& 3.5e-5$\pm$2.1e-5 & {1.6e-5}
			& 27.60$\pm$4.3e-2 & {27.60}
			&{\ding{51}\ding{55}} \\
			\midrule
			\multirow{2}{*}{[1e-2,1e-3,5e-5]} & PPO-EAL
			& 7.2e-3$\pm$3.4e-4 & 8.1e-3
			& 9.4e-4$\pm$6.1e-5 & 9.1e-4
			& 4.2e-5$\pm$5.0e-6 & 4.6e-5
			& 27.54$\pm$2.8e-2 & 27.58
			& {\ding{51}} \\
			{}& PPO-EAL-m
			& 8.4e-3$\pm$3.9e-4 & 8.7e-3 
			& 8.1e-4$\pm$7.5e-5 & 7.6e-4
			& 3.6e-5$\pm$2.0e-5 & {1.0e-5}
			& 27.59$\pm$5.5e-2 & {27.63}
			&{\ding{51}} \\
			\midrule
			\multirow{2}{*}{[1e-2,1e-3,1e-5]} & PPO-EAL
			& 1.0e-2$\pm$3.1e-4 & 9.3e-3
			& 8.1e-4$\pm$9.2e-5 & 7.5e-4
			& 4.8e-5$\pm$1.7e-5 & \red{2.8e-5}
			& 27.45$\pm$5.1e-2 & 27.53
			& {\ding{55}} \\
			{}& PPO-EAL-m
			& 8.7e-3$\pm$2.0e-4 & 8.6e-3 
			& 9.5e-4$\pm$1.2e-4 & 9.7e-4
			& 9.9e-5$\pm$3.9e-5 & \red{8.1e-5}
			& 27.51$\pm$6.4e-2 & {27.56}
			&{\ding{55}} \\
			\midrule
			\multirow{2}{*}{[1e-3,1e-3,1e-4]} & PPO-EAL
			& 4.9e-4$\pm$1.2e-4 & 3.4e-4
			& 6.8e-4$\pm$7.7e-5 & 5.7e-4
			& 2.4e-4$\pm$2.2e-5 & \red{1.9e-4}
			& 27.25$\pm$3.5e-2 & 27.27
			& {\ding{55}} \\
			{}& PPO-EAL-m
			& 8.1e-4$\pm$1.0e-4 & 7.1e-4 
			& 9.0e-4$\pm$4.7e-5 & 8.7e-4
			& 1.1e-4$\pm$2.2e-5 & {1.0e-4}
			& 27.30$\pm$3.8e-2 & {27.34}
			&{\ding{51}} \\
			\midrule    
			\multirow{2}{*}{[1e-3,1e-3,5e-5]} & PPO-EAL
			& 3.9e-4$\pm$3.8e-5 & 3.8e-4
			& 6.7e-4$\pm$6.2e-5 & 6.3e-4
			& 1.9e-4$\pm$1.4e-5 & \red{2.1e-4}
			& 27.26$\pm$3.2e-2 & 27.31
			& {\ding{55}} \\
			{}& PPO-EAL-m
			& 7.0e-4$\pm$9.4e-5 & 7.5e-4 
			& 7.7e-4$\pm$4.1e-5 & 7.4e-4
			& 3.1e-5$\pm$1.1e-5 & {2.6e-5}
			& 27.24$\pm$6.8e-2 & {27.30}
			&{\ding{51}} \\
			\midrule
			\multirow{2}{*}{[1e-3,1e-3,1e-5]} & PPO-EAL
			& 5.1e-4$\pm$1.7e-4 & 5.1e-4
			& 7.0e-4$\pm$1.1e-4 & 7.8e-4
			& 1.6e-4$\pm$2.6e-5 & \red{1.7e-4}
			& 27.14$\pm$7.7e-2 & 27.20
			& {\ding{55}} \\
			{}& PPO-EAL-m
			& 7.1e-4$\pm$1.2e-4 & 6.9e-4 
			& 8.0e-4$\pm$5.6e-5 & 8.2e-4
			& 4.5e-5$\pm$1.4e-5 & \red{2.9e-5}
			& 27.28$\pm$2.7e-2 & {27.30}
			&{\ding{55}} \\      
			\bottomrule
		\end{tabular}%
	}
\end{table*}

\begin{figure}
	\centering
	\includegraphics[width=0.49\textwidth]{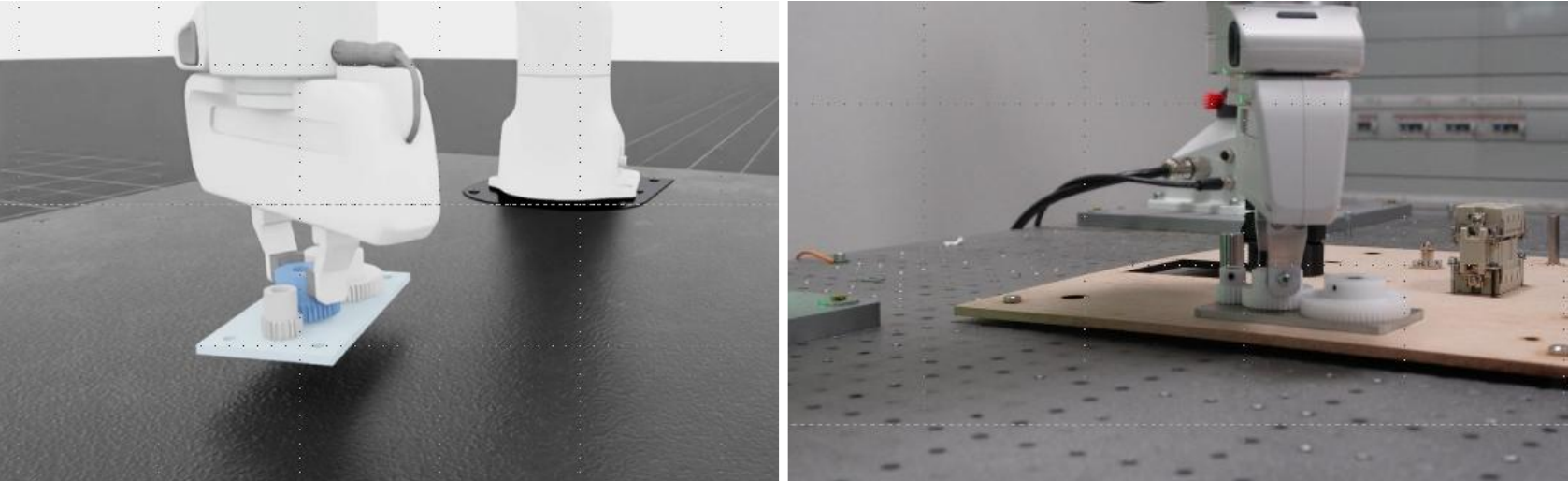}
	\caption{Contact-rich gear assembly with PPO-EAL policy: (left) simulation evaluation, (right) hardware deployment. }
	\label{fig:gear_mesh_task}
\end{figure}

\subsection{Real-world Application--Contact-rich Gear Assembly}
To further assess the effectiveness of the proposed method, we apply it to a contact-rich gear assembly task, which is substantially more challenging due to the stringent precision and control requirements imposed by restrictive contact conditions.  The gear assembly task can be divided into two stages. In the first stage, the gripper moves the gear toward the target shaft position and aligns it with the shaft axis. In the second stage, the gripper applies downward insertion force while simultaneously rotating the gear to achieve successful tooth engagement and complete the assembly process. Task success is determined by satisfying geometric conditions on XY alignment, insertion depth, and final alignment between the gear and target shaft.

\subsubsection{RL formulation} The RL framework adopts an asymmetric actor–critic structure. The actor observation consists of fingertip position, fingertip position relative to the fixed target, fingertip orientation quaternion, end-effector linear velocity, angular velocity, 6D hand wrench, and the last action. The critic state contains all actor observations plus privileged information such as joint positions, manipulated object pose, target object pose, controller gains, and control thresholds. 

The action space is a 6D vector, where the first three dimensions are the incremental Cartesian position displacement and the last three dimensions are the incremental rotational displacement of the end-effector. Actions are scaled by predefined positional and rotational thresholds before being converted into target poses for task-space impedance controller.

The reward function is composed of multi-scale keypoint-based alignment rewards, action magnitude penalties, action smoothness penalties, engagement bonuses, and final success bonuses. This reward structure progressively encourages object approach, alignment, and successful insertion while discouraging excessive or unstable motions.
Each episode terminates when the maximum episode length is reached.  

As described above, PPO-EAL and PPO employ identical network architectures with the same depth and width. For safety considerations, the contact force threshold is set to 15 N, with the desired violation rate constrained below 2e-3. During training, the target horizontal shaft position is randomized within $\pm$5\,cm of [0.35, 0.25]\,m. For sim-to-real transfer, domain randomization is applied by introducing noisy position observations and varying control gains. 
\subsubsection{Results}
We compared PPO-EAL and PPO through real-world experiments. The results over 20 hardware tests show that the PPO-EAL achieves an 80\% success rate, while the na{\"i}ve PPO only reaches a 30\% success rate.

For successful trials only, we further evaluate the mean values and standard deviations of the following metrics: (1) the peak contact force (${F_{\max}}$) and peak torque (${\tau_{\max}}$) encountered during task execution, (2) the task completion time ($T_s$), defined as the elapsed time from the start of policy execution to the first time the measured end-effector reaches its lowest insertion height, (3) the average XY-plane position error between the reference and measured end-effector positions before task completion ($e_{\text{track}}$), which reflects excessive lateral motion during task execution, and (4) the average XY-plane distance between the reference trajectory and the target insertion position before task completion ($e_{\text{h}}$), which characterizes the reference policy’s insertion efficiency and search behavior.
Results in Table~\ref{tab:result_gear_mesh_task} show that, compared with PPO, PPO-EAL achieves substantially safer gear assembly behavior by reducing peak contact force by approximately 46\%. Also, the mean values of average $e_\text{track}$ and $e_\text{h}$ across successful trials using PPO-EAL are smaller than those using PPO, demonstrating reduced lateral motion during task execution and more effective search movement. Although PPO-EAL requires longer task completion time, the results demonstrate a favorable safety–performance tradeoff. 

Furthermore, compared with PPO-EAL, the momentum-regulated PPO-EAL-m further improves real-world deployment performance by reducing peak contact force, peak torque, task completion time, while achieving almost the same level of success rate and lateral tracking errors. In particular, PPO-EAL-m achieves the lowest peak force and torque among all evaluated methods, demonstrating that momentum regulation translates into superior practical safety-performance trade-offs in contact-rich robotic manipulation.

We also tested the generalization capabilities of PPO-EAL, subjecting varying target shart positions and gear mass. Comparison tests can be found in the supplementary video. While PPO is not able to generalize to new situations (the success rate over 20 tests drops to 0\%), we found that the PPO-EAL policy enables generalization to new situations (success rate $\approx$ 80\% over 20 runs) where the target shaft position goes beyond the zone experienced in the training stage, 
demonstrating robustness in real-world deployment.

\begin{table}
	\centering
	\renewcommand{\arraystretch}{1.25}
	\caption{Gear assembly performance comparison.}
	\label{tab:result_gear_mesh_task}  
	\resizebox{0.49\textwidth}{!}{
		\begin{tabular}{cccccc}
			\toprule 
			{}&$F^{\max}$\,[N]& $\tau^{\max}$\,[Nm] & $T_s$\,[s] & $e_\text{track}$\,[m] &$e_\text{h}$\,[m] \\
			\midrule
			{PP0}& 53.4$\pm$8.4 &3.0$\pm$0.7 & 9.1$\pm$4.1 & 1.2e-2$\pm$2.6e-3 & 1.0e-2$\pm$1.7e-4 \\ 
			PP0-EAL& 28.7$\pm$5.7 &3.0$\pm$0.5 & 11.3$\pm$6.1 & 7.9e-3$\pm$9.6e-4 & 9.7e-3$\pm$4.8e-4 \\    
			PP0-EAL-m& 21.5$\pm$4.8 &2.6$\pm$0.3 & 9.3$\pm$5.2 & 8.0e-3$\pm$6.9e-4 & 9.9e-3$\pm$3.6e-4 \\              
			\bottomrule
		\end{tabular}
	}
\end{table}

\section{Conclusion and Discussion}\label{conclu_sec}
This work presents PPO-EAL, a constrained proximal policy optimization framework that integrates exact augmented Lagrangian optimization into first-order deep reinforcement learning for safe robotic control. By combining clipped policy optimization with exact quadratic penalty terms and stabilized dual-variable updates, PPO-EAL achieves precise safety constraint satisfaction while maintaining strong task performance across diverse robotic systems. Extensive validation in GPU-accelerated benchmark tasks, including inverted cart-pole balancing, cart-double pendulum stabilization, Franka manipulator control, and quadrupedal locomotion, demonstrates that PPO-EAL consistently outperforms existing first-order safe RL baselines in terms of safety precision and learning stability. In particular, real-world hardware experiments confirm that PPO-EAL substantially improves contact safety, and sim-to-real generalization for challenging contact-rich manipulation tasks, highlighting its practical value for safety-critical control.

Despite these promising results, several directions remain for future research. First, while the current implementation primarily relies on feedforward actor–critic architectures, more advanced policy structures such as recurrent networks may further improve performance in robot control tasks requiring richer temporal reasoning. 
Second, the CMDP framework in the current work only provides soft constraints in safety control. Since no online fine-tuning is considered, the sim2real performance depends heavily on the high-fidelity simulation. In the future, we will incorporate the safety shielding into the learning framework, providing rigorous safety guarantees.



\bibliography{sty/ref.bib}
\bibliographystyle{sty/myIEEEtran}

\end{document}